\pdfoutput=1
\documentclass{article}
\usepackage[T1]{fontenc}

\usepackage[left=31mm,right=31mm,top=33mm,bottom=20mm]{geometry}
\usepackage{natbib}
\usepackage{authblk}
\usepackage{amsmath,amssymb,amsfonts}
\usepackage{amsthm}

\newtheorem{theorem}{Theorem}
\newtheorem{lemma}{Lemma}
\newtheorem{corollary}{Corollary}
\newtheorem{proposition}{Proposition}

\newtheorem{definition}{Definition}

\usepackage{enumitem}

\usepackage[dvipsnames]{xcolor}
\usepackage[colorlinks,
            linkcolor=blue,
            citecolor=blue,
            urlcolor=magenta,
            linktocpage,
            plainpages=false]{hyperref}

\bibpunct{(}{)}{;}{a}{,}{,}

\usepackage{booktabs} % For formal tables
\usepackage[lined,ruled,noend]{algorithm2e} % For algorithms 

\SetAlFnt{\small}
\SetAlCapFnt{\small}
\SetAlCapNameFnt{\small}
\SetAlCapHSkip{0pt}
\IncMargin{-\parindent}

\usepackage{booktabs}
\usepackage{graphicx}

\usepackage{bm}

\usepackage{multirow,array}

\DeclareMathOperator*{\nth}{^{\text{th}}}

\DeclareMathOperator*{\argmax}{arg\,max}
\DeclareMathOperator*{\argmin}{arg\,min}

\DeclareMathOperator{\E}{\mathbb{E}}

\renewcommand{\Pr}{\operatorname{\mathbb{P}}\nolimits}
\newcommand{\ind}[1]{\mathop{{\bf 1}\left[#1\right]}}

\newcommand{\reals}{\mathbb{R}}

\newcommand{\cF}{\mathcal{F}}

\newcommand{\loss}{\ell}

\newcommand{\bigmid}{\mathrel{}\middle|\mathrel{}}

\newcommand{\Woe}{W}
\newcommand{\woe}{w} 

\newcommand{\cE}{\mathcal{E}}

\newcommand{\elfRH}{\textsc{FPL-ELF}}
\newcommand{\elfFH}{\textsc{FPL-SELF}} % S for Stable or Static
\newcommand{\banditelfRH}{\textsc{FPL-ELF-$\varepsilon$}}
\newcommand{\boldbanditelfRH}{\textsc{FPL-ELF-$\bm{\varepsilon}$}}
\newcommand{\banditelfFH}{\textsc{FPL-SELF-$\varepsilon$}}
\newcommand{\fh}{\textsc{self}}

\newcommand{\gelfRH}{\textsc{General-FPL-ELF}}

\usepackage{cancel}

\newcommand{\bmin}{\lambda_{\min}}
\newcommand{\bmax}{\lambda_{\max}}
\newcommand{\bzero}{\lambda_0}
\newcommand{\babs}{\lambda}

\newcommand{\diff}{\mathrm{d}} % differentiate
\newcommand{\Rin}{R_{\mathrm{in}}} 
\newcommand{\Rout}{R_{\mathrm{out}}} 
\newcommand{\lPoi}{l_{\mathrm{Poi}}} 
\newcommand{\uPoi}{u_{\mathrm{Poi}}}
\newcommand{\locallPoi}{l'_{\mathrm{Poi}}}
\newcommand{\localuPoi}{u'_{\mathrm{Poi}}}
\newcommand{\Przero}{\Pr_0}
\newcommand{\Ezero}{\E_0}

\newcommand{\strictlytruthful}{\text{strictly truthful}}

\newcommand{\truthful}{\text{truthful}}

\newcommand{\truthfulness}{\text{truthfulness}}
\newcommand{\Truthfulness}{\text{Truthfulness}}

\title{No-Regret Incentive-Compatible Online Learning under Exact Truthfulness with Non-Myopic Experts}

\author{
Junpei Komiyama\thanks{Stern School of Business, New York University, \texttt{junpei@komiyama.info}}
\qquad
Nishant A.~Mehta\thanks{Computer Science, University of Victoria, \texttt{nmehta@uvic.ca}}
\qquad
Ali Mortazavi\thanks{Computer Science, University of Victoria, \texttt{alithemorty@gmail.com}}
}
\begin{document}

\maketitle

\begin{abstract}
We study an online forecasting setting in which, over $T$ rounds, $N$ strategic experts each report a forecast to a mechanism, the mechanism selects one forecast, and then the outcome is revealed. In any given round, each expert has a belief about the outcome, but the expert wishes to select its report so as to maximize the total number of times it is selected. 
The goal of the mechanism is to obtain low belief regret: the difference between its cumulative loss (based on its selected forecasts) and the cumulative loss of the best expert in hindsight (as measured by the experts' beliefs). 
We consider exactly truthful mechanisms for non-myopic experts, meaning that truthfully reporting its belief strictly maximizes the expert's subjective probability of being selected in any future round. Even in the full-information setting, it is an open problem to obtain the first no-regret exactly truthful mechanism in this setting. 
We develop the first no-regret mechanism for this setting via an online extension of the Independent-Event Lotteries Forecasting Competition Mechanism (I-ELF). By viewing this online I-ELF as a novel instance of Follow the Perturbed Leader (FPL) with noise based on random walks with loss-dependent perturbations, we obtain $\tilde{O}(\sqrt{T N})$ regret. Our results are fueled by new tail bounds for Poisson binomial random variables that we develop. We extend our results to the bandit setting, where we give an exactly truthful mechanism obtaining $\tilde{O}(T^{2/3} N^{1/3})$ regret; this is the first no-regret result even among approximately truthful mechanisms.
\end{abstract}

\section{Introduction}

A forecasting competition is an information-elicitation mechanism 
that takes in $N$ experts' probability forecasts for multiple outcomes and awards a cash prize to a single expert.
Supposing that 
each expert has a belief about (i.e., a probability distribution over) each outcome, 
the mechanism wishes to maximize the probability of awarding the prize to the expert with the best \emph{beliefs}, regardless of what the expert actually reported. More precisely, 
given a loss function that assigns a loss to any probablity forecast under any outcome realization, the goal of the mechanism is to maximize the probability that it awards the prize to an expert whose beliefs have cumulative loss that is approximately minimal. 
Recent works have significantly advanced our understanding of the design of forecasting competition mechanisms, showing both \emph{(exactly) truthful} mechanisms \citep{witkowski2018incentive,witkowski2023incentive} and \emph{approximately truthful} mechanisms \citep{frongillo2021efficient}, complete with event complexity guarantees; 
briefly, the \emph{event complexity} of a mechanism is the number of outcomes required to guarantee that the mechanism outputs, with some fixed confidence $1 - \delta$, an expert whose beliefs induce cumulative loss within $\varepsilon$ of the minimum.

Forecasting competitions naturally extend to the online setting where an online mechanism maintains a probability distribution over the experts. Here, in each of $T$ rounds: the experts forecast one outcome, the mechanism forms its own report by using its distribution over experts (e.g., selecting one expert's report at random), and the outcome is finally revealed. The simplest online payment rule selects a single winner after the final round.\footnote{Alternatively, one can award a prize at the end of each round; our work encompasses this setting as well.} The goal of the mechanism is to obtain low \emph{belief regret}, which is the difference between the cumulative loss of the mechanism (based on its reports) and the cumulative loss of the best expert in hindsight (as measured by the expert's beliefs). This online setting was recently studied by \citet{freeman2020no} and \citet{frongillo2021efficient} (a related setting\footnote{\citet{roughgarden2017online} also considered this problem but in a slightly, yet importantly, different setting where the incentive of a (myopic) expert is to maximize its unnormalized weight (assigned by a weighted-based mechanism), whereas we consider the incentive to be the probability of being selected.} was introduced earlier by \citet{roughgarden2017online}). However, while \citet{freeman2020no} briefly discussed the case of non-myopic experts, their belief regret bounds were restricted to myopic experts rather than experts that aim to maximize their total future payoff; the myopic experts assumption is particularly problematic when a prize is only awarded after the final round of the game. \citet{frongillo2021efficient} do consider non-myopic experts, but they focus on approximately (strictly) truthful mechanisms and use a weaker notion of incentive compatibility as compared to \citet{witkowski2023incentive}. 
\Truthfulness{}\footnote{We use ``\truthful{}'' to mean honest reporting is an optimal strategy and ``\strictlytruthful{}'' to mean honest reporting is a uniquely optimal strategy.} is arguably of interest in its own right: if an expert has no incentive to spend resources on deciding how to strategically report, it may instead use those resources in other ways, including investing in obtaining better-informed beliefs. 
It is an open problem of some importance \cite[Section 5]{freeman2020no} to obtain the first mechanism that is no-regret and \truthful{} for non-myopic experts. 

In this work, we resolve the open problem: we present the first no-regret mechanisms for the online forecasting competition setting that are \truthful{} for non-myopic experts; in fact, our mechanisms are \strictlytruthful{}. We do this not only for the full-information setting --- which was described above --- but also for the bandit feedback setting that we soon describe below. 
In the bandit setting, nothing was previously known about no-regret learning for non-myopic experts, even when considering mechanisms that are not \truthful{}; 
prior work by \citet{freeman2020no} and \citet{zimmert2024productive} was restricted to myopic experts.

Our mechanism for the full-information setting, \elfRH{}, is essentially an online extension of the Independent-Event Lotteries Forecasting Competition Mechanism (I-ELF) of \citet{witkowski2023incentive}; the mechanism ELF-X, also essentially an online extension of I-ELF, was already proposed as a candidate no-regret algorithm (without regret analysis) by \citet{freeman2020no}. 
We show that the expected regret of \elfRH{} is of order $O(\sqrt{T N} \log T)$ for $T \ge N$. 
Unfortunately, we suspect that a dependence of at least $\Omega(\sqrt{T N})$ is fundamental for any \truthful{} mechanism (under non-myopic experts), although at present we lament that we do not have a lower bound; we discuss our conjecture in more detail in Section~\ref{sec:discussion}. 
Conceptually, a major contribution of our work is showing how the regret of this online extension of I-ELF can be viewed as a form of Follow the Perturbed Leader, a fundamental online learning algorithm paradigm dating back to \citet{hannan1957approximation}. Specifically, we show that \elfRH{} is a sophisticated generalization of the algorithm Prediction by Random-Walk Perturbation \citep{devroye2013prediction}. Our most major technical contribution is showing how to extend the ingenious regret analysis of \citet{devroye2013prediction} to bound the regret of \elfRH{}. Whereas the prior analysis could proceed via concentration and anti-concentration results for binomial random variables, our analysis is significantly more challenging due to the need to establish similar results for Poisson binomial random variables.

In the bandit setting, in each round, the mechanism selects a single expert and observes the report only of that expert. Consequently, once the outcome for the round is revealed, the mechanism only knows the loss of the selected expert. In this setting, it becomes natural to pay an expert a fixed sum of money in each round it is selected. Pragmatically, only the selected expert needs to be aware of its belief, a useful prospect if experts form their beliefs based on investing in research. 
Our mechanism \banditelfRH{} --- an \emph{exploration-separated} version of \elfRH{} --- is the first-ever no-regret learning algorithm for non-myopic experts under bandit feedback. 
Moreover, just like our mechanism for the full-information setting, our mechanism for the bandit setting is \truthful{} for non-myopic experts. 
Our bound on the expected regret is of order $O\left( T^{2/3} N^{1/3} \log T \right)$. The $T^{2/3}$ appears (rather than $\sqrt{T}$) because our mechanism is exploration-separated; we believe that this dependence on $T$ is unimprovable for any \truthful{} mechanism. As we discuss in Section~\ref{sec:discussion}, the $T^{2/3}$ factor can be mitigated if the forecaster receives an additional report at each round.

The sequel of the paper progresses as follows. The next section formally presents the problem setting. We then present our results for the full-information setting in Section~\ref{sec:full-info} and the bandit setting in Section~\ref{sec:bandit}. In Section~\ref{sec:regret-analysis}, we provide a sketch of our regret analysis for both feedback settings. Finally, we conclude in Section~\ref{sec:discussion} with a discussion of extensions and open problems.

\section{The Incentive-Compatible Online Forecasting Game}

We now formally present the problem of incentive-compatible online forecasting with non-myopic experts, starting with the case of full-information feedback. This game involves $N$ experts, each of whom has beliefs related to $T$ binary\footnote{It is straightforward to extend all our results to the case of finite outcome spaces.} outcomes that will be revealed sequentially. The actual structure of the beliefs --- which we inherit from \citet{witkowski2023incentive} --- is sophisticated; we defer a full description of the belief structure until later in this section. 
In each round $t$, each expert $j$ provides a probability forecast, or \emph{report}, $r_{j,t} \in [0, 1]$ for that round's outcome. An online mechanism $M$ maintains a probability distribution over the experts, assigning probability $p_{j,t}$ to expert $j$ in round $t$; here, $p_{j,t}$ is based upon all the information revealed prior to round $t$ as well as any internal randomization used by the mechanism thus far. In each round, the mechanism plays its own report $\hat{r}_t$ either by selecting an expert $I_t \sim p_t$ at random and setting $\hat{r}_t = r_{I_t,t}$ or by hedging to form an aggregated report $\hat{r}_t = \sum_{j=1}^N p_{j,t} r_{j,t}$. 

\begin{algorithm}[t]
\DontPrintSemicolon
For $j \in [N]$, Nature selects belief distribution $P^{(j)} \in \left( \Delta \bigl( \{0, 1\} \times [0,1]^{N-1} \bigr) \right)^T$ \;
Mechanism initializes $p_1$\;
\For{$t\leftarrow 1$ \KwTo $T$}{
  Mechanism draws $I_t \sim p_t$ \;
  Expert $I_t$ gets payoff \$1 \;
  \emph{Full-information feedback:} \;
  \quad Each expert $j \in [N]$ selects and reveals report $r_{j,t} \in [0, 1]$ \;
  \emph{Bandit feedback:} \;
  \quad Expert $I_t$ selects and reveals report $r_{I_t,t} \in [0, 1]$ \;
  Mechanism selects report $\hat{r}_t \in [0, 1]$ \;
  Nature reveals outcome $o_t \in \{0, 1\}$ and Mechanism suffers loss $\loss(\hat{r}_t, o_t)$ \;
  Mechanism sets $p_{t+1}$\;
}
Expert $I_{T+1}$ gets payoff \$1\;
  
\caption{Online Forecasting Protocol with Full-Information or Bandit Feedback}
\end{algorithm}

We adopt incentives of a general form that was introduced by \citet{frongillo2021efficient}. 
In round $t$, each expert $j$ wishes to maximize the incentive
\begin{align} 
\sum_{t'=t+1}^{T+1} \alpha_j(t' \mid t) \ind{I_{t'} = j} , \label{eqn:general-incentive}
\end{align}
where, as before, $I_t \sim p_t$. 
In the above expression, 
we assume, for all $j \in [N]$ and all $t \in [T]$, that $\alpha_j(t' \mid t) \geq 0$
for all $t' \in\{t + 1, t + 2, \ldots, T + 1\}$ with the inequality being strict for at least one value of $t'$. 
Note that $\alpha_j(t' \mid t)$ is the value that expert $j$, when in round $t$, assigns to being selected in some future round $t'$. As mentioned by \citet{frongillo2021efficient}, the form \eqref{eqn:general-incentive} can capture time-inconsistent preferences since $\alpha_j(t' \mid t)$ can vary freely with $t$.

A myopic expert \citep{freeman2020no,zimmert2024productive} is a special version of such an expert that wishes to maximize the incentive for the next round. Namely, 
$\alpha_j(t' \mid t)$ is $1$ for $t' = t+1$ and is zero for any other $t'$. 
Another type expert that may be of importance is an ``all's well that ends well'' expert, for whom $\alpha_j(t' \mid t)$ is $1$ for $t' = T+1$ and is zero for any other $t'$. 
More generally, there can be other ``single-target'' experts that target just one round, such as $\alpha_j(t' \mid t)$ being $1$ for some particular $t' > t$ and zero anywhere else.

In our case, we consider a truly non-myopic expert that can target any round $t'$. Namely, our goal is to provide a guarantee for any $(\alpha_j(t' \mid t))_{t,t'}$. A sufficient and necessary condition of incentive compatibility across all non-myopic experts is equivalent to simultaneously ensuring incentive compatibility for all single-target experts. This is because maximizing the incentives at any future round implies maximizing the incentives for any weighted sum of the future rounds. 

One important aspect of each expert $j$'s belief for a given round $t$ is that the expert maintains a marginal distribution, specified by success probability $b_{j,t} \in [0, 1]$, over that round's binary outcome $o_t$. We assume that the experts' beliefs are determined by an oblivious adversary and hence are fixed before the game begins. Let $\loss \colon [0, 1] \times \{0, 1\} \rightarrow [0, 1]$ be a loss function that is \emph{strictly proper} \citep{gneiting2007strictly,buja2005loss}. Briefly, a loss function is strictly proper if, for all beliefs $b \in [0, 1]$ and all reports $r \in [0, 1]$ such that $r \neq b$, we have
\begin{align*}
\E_{o \sim \mathrm{Bernoulli}(b)} \left[ \loss(b, o) \right] 
< \E_{o \sim \mathrm{Bernoulli}(b)} \left[ \loss(r, o) \right] .
\end{align*}
 The goal of the mechanism is to achieve low \emph{belief regret}, defined as the mechanism's regret against the best expert in hindsight when that expert is evaluated according to its internal beliefs:
\begin{align*}
R_T = \sum_{t=1}^T \loss(\hat{r}_t, o_t) - \min_{j \in [N]} \sum_{t=1}^T \loss(b_{j,t}, o_t) .
\end{align*}

Naturally, if experts wildly misreport their beliefs, meaning that $r_{j,t}$ can be far from $b_{j,t}$, in general the mechanism has no hope of guaranteeing low belief regret. An important subclass of mechanisms are those that are \emph{\strictlytruthful{}}, 
meaning that it is strictly in an expert's best interest (according to its beliefs) to report honestly, i.e., to select $r_{j,t} = b_{j,t}$. To formalize this, we introduce our online generalization of a notion of incentive compatibility that was first introduced by \citet{witkowski2023incentive}.

\subsection{Online incentive compatibility under belief independence}

Focusing on the one-shot game setting, \citet{witkowski2023incentive} introduced a strong notion of incentive compatibility that they called \emph{incentive compatibility under belief independence}, itself a relaxation of the concept of \emph{robust incentive compatibility} that they also introduced. We will extend incentive compatibility under belief independence to the online setting. Let us first revisit the definitions for the one-shot game. 

\begin{definition}[Belief Independence \citep{witkowski2023incentive}] \label{def:belief-indep}
For any expert $i \in [N]$, let $D^{(i)}$ be a joint probability distribution over the outcome sequence $o_1, \ldots, o_T$ and the sequence of reports of the other experts $r_{-i,1}, \ldots, r_{-i,T}$. We say that $D^{(i)}$ is \emph{belief independent} if it factorizes across rounds; that is, $D^{(i)}$ is a product of distributions $D^{(i)}_t$, each of which specifies expert $i$'s distribution over $(o_t, r_{-i,t})$.
\end{definition}

\begin{definition}[Incentive Compatibility under Belief Independence \citep{witkowski2023incentive}] \label{def:one-shot-game-ic}
We say that a mechanism $M$ is \emph{strictly incentive compatible under belief independence} (hereafter, ``IC-BI'') 
if, for all experts $i \in [N]$ and all belief independent distributions $D^{(i)}$ over both the outcome sequence $o_{1:T}$ and reports of the other experts $r_{-i,1:T}$, the following holds:
For all reports $r_{i,1:T} \neq b_{i,1:T}$ of expert $i$ that are nontruthful in at least one round,
\begin{align}
\Pr_{r_{-i,1:T}, o_{1:T} \sim D^{(i)}} 
\left(
  M(r_{-i,1:T}, b_{i,1:T}, o_{1:T}) = i 
\right) 
> 
\Pr_{r_{-i,1:T}, o_{1:T} \sim D^{(i)}} 
\left(
  M(r_{-i,1:T}, r_{i,1:T}, o_{1:T}) = i 
\right) . \label{eqn:ic-bi}
\end{align}
\end{definition}

Another notion of incentive compatibility commonly considered in the literature is \emph{immutable-belief incentive compatibility} \citep{lambert2015axiomatic,chen2019randomized,frongillo2021efficient}, in which experts model only the outcomes (and not the other experts' reports). It is not hard to see that IC-BI is strictly stronger than immutable-belief incentive compatibility, meaning that all IC-BI mechanisms are immutable-belief incentive compatible while the reverse is not true in general (see \cite[Appendix A]{witkowski2023incentive}). Note that all our results are for (online extensions of) IC-BI.

Shifting to the online setting, an online mechanism is essentially a mechanism that allows for a flexible stopping time. 
\begin{definition}[Online Mechanism] \label{def:online-mechanism}
We say that $M$ is an \emph{online mechanism} if it is a randomized mapping (that is, having its own internal randomization) of the form
\begin{align*}
M \colon \bigcup_{T' \in \{0, 1, \ldots, T\}} \left( [0, 1]^N \times \{0, 1\} \right)^{T'} \rightarrow [N] .
\end{align*}
\end{definition}

Because we consider an online notion of incentive compatibility, our definitions below need to hold for a flexible time horizon $T' \leq T$. Indeed, if an expert must report truthfully in an arbitrary round $t$ in order to maximize their expected payment given just after an arbitrary future round $T'$, then it follows that in every round, truthfully reporting strictly maximizes an expert's expected future cumulative payoff. We now formalize this notion of incentive compatibilty; note that the bandit setting calls for a slight adjustment compared to the full-information setting since it only makes sense to condition on the past reports that were actually made.

\begin{definition}[Online Incentive Compatibility under Belief Independence] \label{def:ic}
In the full-information setting, we say that an online mechanism $M$ is \emph{online strictly incentive compatible under belief independence} (hereafter, ``Online IC-BI'') if, for all experts $i \in [N]$ and all belief independent distributions $D^{(i)}$ over both the outcome sequence $o_{1:T}$ and reports $r_{-i,1:T}$ of the other experts, the following holds:

For all time horizons $T' \leq T$, all rounds $t$ with any history $r_{1:t-1}, o_{1:t-1}$, for all reports $r_{i,t:T'} \neq b_{i,t:T'}$ of expert $i$ that are nontruthful in at least one upcoming round,
\begin{align}
\begin{aligned}
&
\Pr_{r_{-i,t:T'}, o_{t:T'} \sim D^{(i)}} 
\left(
  M(r_{-i,1:T'}, r_{i,1:t-1}, b_{i,t:T'}, o_{1:T'}) = i \mid r_{1:t-1}, o_{1:t-1} 
\right) \\
&> 
\Pr_{r_{-i,t:T'}, o_{t:T'} \sim D^{(i)}} 
\left(
  M(r_{-i,1:T'}, r_{i,1:t-1}, r_{i,t:T'}, o_{1:T'}) = i \mid r_{1:t-1}, o_{1:t-1} 
\right) . 
\end{aligned}
\label{eqn:online-ic-bi}
\end{align}
\end{definition}

Just as IC-BI is strictly stronger than immutable-belief incentive compatibility, Online IC-BI is strictly stronger than the analogously defined online version of immutable-belief incentive compatibility (assuming that each expert believes the outcomes are independent).

From a bandit algorithm's perspective, the history at the start of round $t$ consists of $(I_s, r_{I_s,s}, o_s)_{s \in [t-1]}$. However, from an expert's perspective, the history is potentially richer if the expert also somehow obtains knowledge of the algorithm's random variable realizations. As a concrete example, in each round, our algorithm for the bandit setting randomly decides whether that round is an ``exploration round'' or an ``exploitation round'' (see Section~\ref{sec:bandit-alg} for more information). 
The notion of incentive compatibility we present below is robust to the presence of experts who have this richer notion of history. We assume that this additional information consists only of those random variable realizations that will be used by the mechanism in future rounds, as other realizations that are immediately discarded are not part of the algorithm's persistent state. Formally, we suppose that in any round $t$, this information $\sigma_t$ is an element of an information set $\mathcal{S}$. Collectively, we refer to $\sigma_{1:t}$ as the algorithm's \emph{transcript} at the end of round $t$.

\begin{definition}[Bandit Mechanism] \label{def:bandit-mechanism}
For a mechanism $M$, let $\mathcal{S}$ be the algorithm's per-round information set. We say that $M$ is a \emph{bandit mechanism} if it is a randomized mapping (that is, having its own internal randomization) of the form
\begin{align*}
M \colon \bigcup_{T' \in \{0, 1, \ldots, T\}} \left( \mathcal{S} \times [N] \times [0, 1] \times \{0, 1\} \right)^{T'} \rightarrow [N] .
\end{align*}
\end{definition}

In the definition below, let $e_j \in \reals^N$ be the $j\nth$ standard basis vector for any $j \in [N]$.

\begin{definition}[Bandit Online Incentive Compatibility under Belief Independence] \label{def:bandit-ic}
In the bandit setting, we say that an online mechanism $M$ is \emph{bandit online strictly incentive compatible under belief independence} (hereafter, ``Bandit Online IC-BI'') if, for all experts $i \in [N]$ and all belief independent distributions $D^{(i)}$ over the outcome sequence $o_{1:T}$ and reports $r_{-i,1:T}$ of the other experts, the following holds:

For all time horizons $T' \leq T$, all rounds $t$ with any history $(I_s, r_{I_s,s}, o_s)_{s \in [t-1]}$ and transcript $\sigma_{1:t-1}$, for all reports $r_{i,t:T'} \neq b_{i,t:T'}$ of expert $i$ that are nontruthful in at least one upcoming round,
\begin{align}
\begin{aligned}
&
\Pr_t
\left(
  M(
    \sigma_{1:t-1}, 
    I_{1:T'}, 
    ((r_{-i,s}, b_{i,s}) \cdot e_{I_s})_{s \in \{t, \ldots, T'\}}, 
    o_{1:T'}
  ) 
  = i
\right) \\
&> 
\Pr_t 
\left(
  M(
    \sigma_{1:t-1}, 
    I_{1:T'}, 
    ((r_{-i,s}, r_{i,s}) \cdot e_{I_s})_{s \in \{t, \ldots, T'\}}, 
    o_{1:T'}
  ) 
 = i
\right) ,
\end{aligned}
\label{eqn:bandit-online-ic-bi}
\end{align}
where $\Pr_t$ is defined as 
$\Pr_t(\cdot) := \Pr_{r_{-i,t:T'}, o_{t:T'} \sim D^{(i)}} ( \cdot \mid (I_s, r_{I_s,s}, o_s)_{s \in [t-1]}, \sigma_{1:t-1} )$.
\end{definition}

Note that if a mechanism is Bandit Online IC-BI, then for all transcripts $\sigma_{1:t-1}$, even experts with access to $\sigma_{1:t-1}$ are strictly incentivized to report truthfully in round $t$. Mathematically, the strict inequality \eqref{eqn:bandit-online-ic-bi} holds when conditioning on arbitrary $\sigma_{1:t-1}$. 
Now, consider an expert that does not know $\sigma_{1:t-1}$ but has a subjective belief over $\sigma_{1:t-1}$. Such an expert is strictly incentivized to be truthful since strict inequality in \eqref{eqn:bandit-online-ic-bi} still holds even when taking an expectation over $\sigma_{1:t-1}$ on both sides.

As we will see in Sections~\ref{sec:full-info}~and~\ref{sec:bandit} respectively, each of \elfRH{} (for the full-information setting) and \banditelfRH{} (for the bandit setting) is incentive compatible under its appropriate notion of online incentive compatibility.

\section{Full-information Setting} \label{sec:full-info}

For $j \in [N]$ and $t \in [T]$, let $\loss_{j,t}$ denote $\loss(r_{j,t}, o_t)$. 
We now present \elfRH{} (Algorithm~\ref{alg:simpelf-rh}), our algorithm for the full-information setting. In the remainder of this section, we explain how we derived \elfRH{} from I-ELF \citep{witkowski2023incentive}, show that \elfRH{} is online strictly incentive compatible under belief independence, and bound this algorithm's expected regret. 
Our regret analysis proceeds by observing that the expected regret of \elfRH{} is equal to that of \elfFH{} (Algorithm~\ref{alg:simpelf-fh}, a stabilized version of \elfRH{} to be presented below); we then bound the expected regret of \elfFH{}.

\begin{figure}[t]
\begin{minipage}[t]{0.515\linewidth}
\begin{algorithm}[H]
\DontPrintSemicolon
For $j \in [N]$, draw $\Woe_{j,0} \sim \mathrm{Uniform}([-\frac{1}{4 N}, \frac{1}{4N}])$\;
\For{$t\leftarrow 1$ \KwTo $T$}{
  Select expert $\displaystyle I_t = \argmin_{j \in [N]} \sum_{s=0}^{t-1} \Woe_{j,s}$\;
  Observe losses $\loss_{j,t}$ for all $j \in [N]$\;
  For $j \in [N]$, draw $\Woe_{j,0} \sim \mathrm{Uniform}([-1/4, 1/4])$\; 
  \For{$s\leftarrow 1$ \KwTo $t$}{
    Draw $C_s \sim \mathrm{Uniform}([N])$ \;
    For $j \in [N]$, draw $\Woe_{j,s}$ as
    $\begin{cases}
      \mathrm{Bernoulli} \left( \frac{1}{2} + \frac{1}{4} \loss_{j,s} \right) & \text{if } j = C_s \\
      0 & \text{if } j \in [N] \setminus \{C_s\}
    \end{cases}$ \;
  }
}
\caption{\label{alg:simpelf-rh} \elfRH{}}
\end{algorithm}
\end{minipage}
\hfill
\begin{minipage}[t]{0.475\linewidth}
\begin{algorithm}[H]
\DontPrintSemicolon
For $j \in [N]$, draw $\Woe_{j,0} \sim \mathrm{Uniform}([-\frac{1}{4 N}, \frac{1}{4N}])$\;
\For{$t\leftarrow 1$ \KwTo $T$}{
  Select expert $\displaystyle I_t = \argmin_{j \in [N]} \sum_{s=0}^{t-1} \Woe_{j,s}$\;
  Observe losses $\loss_{j,t}$ for all $j \in [N]$\;
  Draw $C_t \sim \mathrm{Uniform}([N])$ \;
  For $j \in [N]$, draw $\Woe_{j,t}$ as 
  $\begin{cases}
    \mathrm{Bernoulli} \left( \frac{1}{2} + \frac{1}{4} \loss_{j,t} \right) & \text{if } j = C_t \\ 
    0 & \text{if } j \in [N] \setminus \{C_s\} 
    \end{cases}$ \;
    \vspace{2.378em}
}
\caption{\label{alg:simpelf-fh} \elfFH{}}
\end{algorithm}
\end{minipage}
\end{figure}

\subsection{From I-ELF to \elfRH{}}

The Independent-Event Lotteries Forecasting Competition Mechanism (I-ELF) is a mechanism for the (one-shot game) forecasting competition setting. When adapted to losses, it works as follows. For each outcome $o_t$, a lottery is run that awards precisely one point (an internal win) to one expert as follows: the probability that expert $j$ gets the point is
\begin{align}
\frac{1}{N} \left( 1 - \loss(r_{j,t}, o_t) + \frac{1}{N-1} \sum_{k \in [N] \setminus \{j\}} \loss(r_{k,t}, o_t) \right) . \label{eqn:I-ELF}
\end{align}
The winner of the competition is the expert with the most points, with ties broken uniformly at random. 
As shown by \citet{freeman2020no}, any mechanism for the one-shot game setting admits a natural extension to the online forecasting competition setting: in each round $t$, run the one-shot game mechanism for the past $t-1$ rounds and set $I_t$ to be the winner of the competition. \citet{freeman2020no} applied this extension to a slightly modified version\footnote{\citet{freeman2020no} slightly modified I-ELF by changing the sum in \eqref{eqn:I-ELF} to be over \emph{all} experts and hence also changing the normalization factor from $\frac{1}{N-1}$ to $\frac{1}{N}$; as we describe in Appendix~\ref{app:elf-x}, this change is necessary for no-regret learning.} of I-ELF and called the result ELF-X.\footnote{In 2020, only the conference version \citep{witkowski2018incentive} was in existence, and in that version the algorithm ``I-ELF'' from \cite{witkowski2023incentive} was called ``ELF''.}

We follow the same approach for extending a mechanism from the one-shot game to the online setting, but we first modify the one-shot game mechanism (I-ELF mechanism) in three ways. The first modification is for simplicity: we drop the summation term in \eqref{eqn:I-ELF}, giving the remaining probability to a ``dummy'' expert. The second modification is for technical reasons: we change the term $1 - \loss(r_{j,t}, o_t)$ to $\frac{1}{2} - \frac{1}{4} \loss(r_{j,t}, o_t)$; if $\Woe_{j,t}$ is the indicator random variable that expert $j$ gets a point for outcome $t$, then our modification ensures that the effect of the loss is to change the variance of $\Woe_{j,t}$ by at most a factor of $4$. Finally, we switch from a traditional ``happy lottery'' to a ``sad lottery'', in which getting a point is a bad thing; we refer to these points as ``sad points'', which may also be thought of as units of ``woe''. In the course of our regret analysis, this change turned out to be essential\footnote{We could instead consider a happy lottery with $N-1$ to $N$ winners, which is equivalent to a sad lottery with at most $1$ loser. We chose the latter as it is conceptually simpler.} (see our discussion in Section~\ref{sec:general-fpl-analysis}). To enact this change, we further change $\frac{1}{2} - \frac{1}{4} \loss(r_{j,t}, o_t)$ to $\frac{1}{2} + \frac{1}{4} \loss(r_{j,t}, o_t)$ and declare the winner of the competition to be the expert with the \emph{least} sad points (equivalently, the least woe), again with ties broken uniformly at random. 
In homage to Jorge Luis Borges, one might dub this sad lottery as a ``Babylonian Lottery'', since winning the lottery can lead to punishment \citep{borges1956ficciones}. 
\elfRH{} (Algorithm~\ref{alg:simpelf-rh}) extends this modified version of I-ELF to the online setting. For our incentive compatibility and regret analyses, we found it convenient to express (for any round $t$) the random variable $\Woe_{j,s}$ via a generative process. This can be seen in Algorithm~\ref{alg:simpelf-rh}: we first draw the \emph{candidate} $C_s$ (the candidate winner of the sad lottery), and, conditional on an expert being the candidate, the expert has a 
$\frac{1}{2} + \frac{1}{4} \loss(r_{C_s,s}, o_s)$ chance of winning the sad lottery.

\subsection{Incentive compatibility} \label{sec:full-info-ic}

We now present our incentive compatibility result for \elfRH{}. 

\begin{theorem} \label{thm:full-info-IC}
\elfRH{} is Online IC-BI.
\end{theorem}

That \elfRH{} is Online IC-BI should not be surprising. Indeed, \citet{witkowski2023incentive} previously showed that I-ELF is IC-BI, and while the style of presentation of our proof is quite different, the high-level ideas are the same. In Appendix~\ref{app:ic}, we present a proof of Theorem~\ref{thm:full-info-IC} out of an abundance of caution combined with our desire for a more explicit incentive compatibility proof than the one provided by \citet{witkowski2023incentive}. In fact, we prove incentive compatibility for a more general class of mechanisms that, in addition to including \elfRH{}, also includes the online extension (as described earlier) of I-ELF.

\subsection{Regret} \label{sec:full-info-regret}

A key insight toward our analysis of \elfRH{}'s expected regret is that the algorithm can be viewed as an instance of Follow the Perturbed Leader (FPL) whose perturbations are defined by a random walk \citep{devroye2013prediction,devroye2015random}. Such FPL algorithms take the form
\begin{align*}
I_t = \argmin_{j \in [N]} \sum_{s=0}^{t-1} \left( \loss_{j,s} + X_{j,s} \right) ,
\end{align*}
where each per-round perturbation $X_{j,s}$ is a zero-mean noise random variable and we adopt the convention that $\loss_{j,0} = 0$. 
Let $\tilde{\loss}_{j,t} := \loss_{j,t} + X_{j,t}$ be the perturbed loss of expert $j \in [N]$ in round $t \in [T]$. 
To see how \elfRH{} is an instance of FPL, observe that since
\begin{align}
\Woe_{j,t} \sim \mathrm{Bernoulli} \left( \frac{1}{N} \left( \frac{1}{2} + \frac{1}{4} \loss_{j,s} \right) \right) , \label{eqn:woe-simple-elf}
\end{align}
the choice $\tilde{\loss}_{j,t} = 4 N \cdot \Woe_{j,t} - 2$ ensures that 
$X_{j,t} = \tilde{\loss}_{j,t} - \loss_{j,t}$ is zero-mean, as required; also, for ``round $0$'' we trivially take $\tilde{\loss}_{j,0} = X_{j,0} = 4 N \cdot \Woe_{j,0}$. 
As we consider a lottery where at most one expert receives a sad point $W_{j,t} = 1$, the random variables $X_{j,t}$ and  $\tilde{\loss}_{j,t}$ are at a scale of $O(N)$. Therefore, the perturbed leader $I_t$ can be expressed as
\begin{align*}
I_t = \argmin_{j \in [N]} \sum_{s=0}^{t-1} \left( \loss_{j,s} + X_{j,s} \right) = \argmin_{j \in [N]} \sum_{s=0}^{t-1} \frac{\loss_{j,s} + X_{j,s} + 2}{4N}
= \argmin_{j \in [N]} \sum_{s=0}^{t-1} W_{j,s} ,
\end{align*}
which is the expert with the least woe, recovering our algorithm's winning criteria. The initial numeric noise $W_{j,0} \in \left[\frac{1}{4 N}, \frac{1}{4 N} \right]$, which is the only non-integer component among $(W_{j,s})_{s=0}^{t-1}$, breaks ties.

Interestingly, unlike all but one previous work \citep{vanerven2014follow}, the per-round perturbations $X_{j,t}$ depend on the losses $\loss_{j,t}$. Also, unlike all but one previous work \citep{li2018sampled}, the noise random variables $X_{1, t}, \ldots, X_{N,t}$ are dependent since $\Woe_{j,t} = 1$ implies that $\Woe_{k,t} = 0$ for all $k \in [N] \setminus \{j\}$. At a high level, we try to follow the regret analysis of \citet{devroye2013prediction}. However, in their work, the random variables $(X_{j,t})_{j \in [N], t \in [T]}$ are i.i.d., with each $X_{j,t}$ taking values $-1/2$ and $1/2$ with equal probability. This noise structure allows them to use concentration and anti-concentration properties of the binomial distribution. In contrast, we essentially have to deal with Poisson binomial random variables since the interesting part of the cumulative perturbed losses are uncentered sums of Bernoulli random variables, each of which can have a different success probability. Also, our noise random variables $X_{j,t}$ are at a much larger scale, requiring additional care in some steps of the analysis.

Our incentive compatibility analysis crucially relies upon redrawing all previous sad point ($\Woe_{j,s}$) random variables each round.\footnote{While we also redraw the candidates ($C_s$), this does not appear to be necessary for \elfRH{} to be Online IC-BI.} 
However, for the regret analysis, it is much simpler to analyze a different algorithm --- dubbed \elfFH{} (``SELF'' stands for ``Stabilized ELF'' or ``Static ELF'') 
and presented in Algorithm~\ref{alg:simpelf-fh} --- which never redraws any previously drawn candidate and sad point random variables. 
Suppose momentarily that all the losses are oblivious (i.e., selected by an oblivious adversary). The reader can verify that in every round, \elfRH{} and \elfFH{} induce the same marginal distribution over the selected expert $I_t$. Consequently, they have the same expected regret; this observation is standard and can also be found in the fundamental work of \citet{kalai2005efficient}. Now, our assumption that all outcomes as well as all experts' beliefs are selected by an oblivious adversary, combined with our result Theorem~\ref{thm:full-info-IC} that \elfRH{} is Online IC-BI, implies (due to truthful reporting) that under \elfRH{} the losses are oblivious. Consequently, since we actually run \elfRH{} (and hence have oblivious losses), to control the expected regret of \elfRH{} it suffices to  analyze the expected regret of \elfFH{} when the latter encounters the aforementioned oblivious losses. 

\elfFH{}'s regret is simpler to analyze because, like previous analyses of the regret of instances of FPL, our approach to bound the regret involves bounding the number of leader changes; here, the term ``leader change'' is shorthand for the \emph{perturbed} leader $I_t$ changing from one round to the next. The stability induced from freezing past random draws should make clear that \elfFH{} should enjoy far fewer leader changes than \elfRH{}.

We now present our expected regret bound for \elfRH{}. We provide an overview of our regret analysis in Section~\ref{sec:regret-analysis}. 

\begin{theorem} \label{thm:full-info-regret}
Assume that each expert's belief distribution satisfies belief independence and that $T \ge N$.
Then the expected regret of \elfRH{} is bounded as
\begin{align*}
\E \left[ \sum_{t=1}^T \loss_{I_t,t} - \min_{j \in [N]} \sum_{t=1}^T \loss_{j,t} \right] 
= O \left( \sqrt{T N} \log T \right) .
\end{align*}
\end{theorem}

The standard (non-incentive compatible) setting can be viewed as our problem setting but where all experts are always honest. In this standard setting, the minimax regret is of order $\Theta(\sqrt{T \log N})$. In contrast, both terms in our regret bound pay a much higher price in terms of $N$. This extra factor comes from the fact that in each round, precisely one expert is selected as the candidate, and only the candidate's loss is used by the algorithm (which is similar to what happens in bandit feedback). Suppose instead that in a round $s$, for each expert $j$ we independently awarded the expert a sad point with probability $\frac{1}{2} + \frac{1}{4} \loss_{j,s}$. Then in the standard setting we could remove the $\sqrt{N}$ factor. However, having these independent draws, all of which involve the same outcome $o_s$, would result in a mechanism that is not Online IC-BI. We discuss this in more detail in Section~\ref{sec:discussion}. 
Regarding the $\log T$ factor, this stems from the analysis of the Poisson binomial tail lower bound (Section \ref{subsec_taillower}). We believe this is an artifact of our analysis. Nevertheless, removing this factor is nontrivial.

\section{Bandit Setting} \label{sec:bandit}

In this section, we introduce our bandit mechanism, \banditelfRH{} (Algorithm~\ref{alg:simpelf-rh.es}).

\subsection[\banditelfRH{}]{\boldbanditelfRH{}} \label{sec:bandit-alg}

\banditelfRH{} is an exploration-separated version of \elfRH{}. 
Exploration-separated algorithms partition the sequence of rounds into 
\emph{exploration rounds} and \emph{exploitation rounds} (see Definition 1.2 of \citet{babaioff2009characterizing}). 
In \banditelfRH{}, the variable $\cE$ collects the exploration rounds. In these rounds, the algorithm selects an expert uniformly at random. Let $\cE_T$ refer to the final value of $\cE$, i.e., \banditelfRH{}'s set of exploration rounds among rounds 1 through $T$. If a round is not an exploration round, then we call it an exploitation round. In an exploitation round, the algorithm uses data collected during previous exploration rounds. 
For \banditelfRH{}, it will often be convenient to use the convention that $\Woe_s = \mathbf{0}$ (i.e., $\Woe_{j,s} = 0$ for all $j \in [N]$) if round $s$ is an exploitation round.

We have the following incentive compatibility result for \banditelfRH{}.
\begin{theorem} \label{thm:bandit-IC}
\banditelfRH{} is Bandit Online IC-BI.
\end{theorem}
The proof of this result is quite similar to the proof of Theorem~\ref{thm:full-info-IC}. For completeness, we give a proof in Appendix~\ref{app:ic-bandit}.

To our knowledge, \banditelfRH{} is the first bandit algorithm for this problem that satisfies a notion of incentive compatibility for nonmyopic experts; even for approximate truthfulness, we are not aware of previous results. The idea of randomly interleaving exploration and exploitation rounds has previously been used for other incentive-compatible bandit problems \citep{babaioff2009characterizing}.

\subsection{Regret} \label{sec:bandit-regret}

Similar to the full-information setting, we will not directly analyze the expected regret of \banditelfRH{}. We instead control its expected regret by upper bounding the expected regret of \banditelfFH{} (Algorithm~\ref{alg:simpelf-fh.es}), a stabilized (or static) version of \banditelfRH{}. 
\banditelfFH{} can also be viewed as an instance of FPL with random-walk perturbation by setting $\Woe_{i,t} = 0$ for all exploitation rounds. Indeed, Section \ref{sec:regret-analysis} provides an analysis that applies to the full-information and bandit settings simultaneously. 
To see how \banditelfFH{} is an instance of FPL, 
let 
$X_{j,t} =
\tilde{\loss}_{j,t} - \loss_{j,t} 
= \frac{4N}{\varepsilon} W_{j,t} - 2 - \loss_{j,t}$. Then we see that
\begin{align*}
I_t \in \argmin_{j \in [N]} \sum_{s=0}^{t-1} \tilde{\loss}_{j,s} = 
\argmin_{j \in [N]} 
\sum_{s=0}^{t-1}
\left(
\frac{4N}{\varepsilon} W_{j,s} 
-
2
\right)
=
\argmin_{j \in [N]}  
\sum_{s=0}^{t-1}
W_{j,s}. 
\end{align*}

For the regret analysis, as we now argue, it suffices to analyze the regret of the \banditelfFH{}. Let $\delta_0$ be the Dirac measure at 0. 
In \banditelfFH{}, round $t$ is an exploitation round if $C_t = 0$; otherwise, it is an exploration round. Let $\cE_T^\fh$ be \banditelfFH{}'s set of exploration rounds among rounds 1 to $T$. 
It is easy to see that the probability that a given round is an exploration round is equal under both algorithms.

\begin{figure}[t]
\begin{minipage}[t]{0.53\linewidth}
\begin{algorithm}[H]
\DontPrintSemicolon
\KwIn{$\varepsilon \in [0, 1]$}
Set $\cE = \{\}$\;
\For{$t\leftarrow 1$ \KwTo $T$}{
  Draw $E \sim \mathrm{Bernoulli}(\varepsilon)$\;
  \uIf{$E = 1$}{
    Draw $C_t \sim \mathrm{Uniform}([N])$\;
    Select expert $I_t = C_t$\;
    Set $\cE = \cE \cup \{t\}$\;
  }
  \Else{
  For $j \in [N]$, draw $\Woe_{j,0} \sim \mathrm{Uniform}([-\frac{\varepsilon}{4N}, \frac{\varepsilon}{4N}])$\;
    \For{$s \in \cE$}{
      For $j \in [N]$, draw $\Woe_{j,s}$ as 
      $\begin{cases}
        \mathrm{Bernoulli} \left( \frac{1}{2} + \frac{1}{4} \loss_{j,s} \right) & \text{if } j = C_s \\ 
        0 & \text{if } j \in [N] \setminus \{C_s\} 
      \end{cases}$ \;
    } 
    Select expert $\displaystyle I_t = \argmin_{j \in [N]} \left\{ \Woe_{j,0} + \sum_{s \in \cE} \Woe_{j,s} \right\}$ \;
  }
}
\caption{\label{alg:simpelf-rh.es} \banditelfRH}
\end{algorithm}
\end{minipage}
\hfill
\begin{minipage}[t]{0.459\linewidth}
\begin{algorithm}[H]
\DontPrintSemicolon
\KwIn{$\varepsilon \in [0, 1]$}
For $j \in [N]$, draw $\Woe_{j,0} \sim \mathrm{Uniform}([-\frac{\varepsilon}{4N}, \frac{\varepsilon}{4N}])$\;
\For{$t\leftarrow 1$ \KwTo $T$}{
  Select expert $\displaystyle I_t = \argmin_{j \in [N]} \sum_{s=0}^{t-1} \Woe_{j,s}$\;
  Draw $C_t \sim (1 - \varepsilon) \cdot \delta_0 + \varepsilon \cdot \mathrm{Uniform}([N])$ \;
  For $j \in [N]$, draw $\Woe_{j,t}$ as 
  $\begin{cases}
    \mathrm{Bernoulli} \left( \frac{1}{2} + \frac{1}{4} \loss_{j,t} \right) & \text{if } j = C_t \\ 
    0 & \text{if } j \in [N] \setminus \{C_t\} 
    \end{cases}$ \;
    \vspace{8.362em}
}
\caption{\label{alg:simpelf-fh.es} \banditelfFH}
\end{algorithm}
\end{minipage}
\end{figure}

Let us see how the expected regret of \banditelfRH{} can be upper bounded via the expected regret of \banditelfFH{}. 
The next proposition (whose proof is in the appendix for completeness) establishes that, conditional on a round being an exploitation round, the marginal distribution of $I_t$ is the same under both algorithms. This allows us to analyze the regret of \banditelfRH{} in exploitation rounds using the behavior of \banditelfFH{} in the same rounds. 
Let $\Pr$ and $\Pr^\fh$ be the probability operators under \banditelfRH{} and \banditelfFH{} respectively.

\begin{proposition} \label{prop:cond-rh-to-fh}
For any $j \in [N]$ and $t \in [T]$,
\begin{align*} 
\Pr \left( I_t = j \mid t \notin \cE_T \right) 
= \Pr^\fh \left( I_t = j \right) .
\end{align*}
\end{proposition}

The reader may have noted that \banditelfFH{}, which is used only in the analysis, violates the bandit feedback model. Indeed, in some rounds, this algorithm selects an expert $I_t$ but uses feedback from another expert $C_t$. The proof of Proposition~\ref{prop:cond-rh-to-fh} shows that despite this violation of bandit feedback, in any round $t$, each algorithm's statistics $\left( \sum_{s=0}^{t-1} \Woe_{j,s} \right)_{j \in [N]}$ have the same law. 
That is to say, the violation of the bandit feedback model for \banditelfFH{} (which never actually needs to be run) does not pose a problem.

The above proposition implies the following relationship between the expected regret of the two algorithms. 
Let $\E$ and $\E^\fh$ be the expectation under \banditelfRH{} and \banditelfFH{} respectively. 
\begin{lemma} \label{lemma:ali-from-A-to-B}
Let the losses be in the range $[0, 1]$. Then
\begin{align*}
\E \left[ \sum_{t=1}^T \loss_{I_t,t} - \min_{j \in [N]} \sum_{t=1}^T \loss_{j,t} \right] 
\leq \E^\fh \left[ \sum_{t=1}^T \loss_{I_t,t} - \min_{j \in [N]} \sum_{t=1}^T \loss_{j,t} \right] + T \varepsilon .
\end{align*}
\end{lemma}

\begin{proof}
Observe that
\begin{align*}
&\E \left[ \sum_{t=1}^T \loss_{I_t,t} \right] \\
&= \E \left[ \sum_{t=1}^T \ind{t \notin \cE_T} \loss_{I_t,t} \right]
      + \E \left[ \sum_{t=1}^T \ind{t \in \cE_T} \loss_{I_t,t} \right] \\
&= \sum_{t=1}^T \Pr(t \notin \cE_T) \cdot \E \left[ \loss_{I_t,t} \mid t \notin \cE_T \right]
      + \sum_{t=1}^T \Pr(t \in \cE_T) \cdot \E \left[ \loss_{I_t,t} \mid t \in \cE_T \right] \\
&= (1 - \varepsilon) \sum_{t=1}^T \E \left[ \loss_{I_t,t} \mid t \notin \cE_T \right]
      + \sum_{t=1}^T \Pr(t \in \cE_T) \cdot \E \left[ \loss_{I_t,t} \mid t \in \cE_T \right] \\
&\leq \sum_{t=1}^T \E \left[ \loss_{I_t,t} \mid t \notin \cE_T \right]
      + T \varepsilon.
\end{align*}
Now, from Proposition~\ref{prop:cond-rh-to-fh}, it holds for all $t \in [T]$ that
\begin{align*}
\E \left[ \loss_{I_t,t} \mid t \notin \cE_T \right] = \E^\fh \left[ \loss_{I_t,t} \right],
\end{align*}
concluding the proof.
\end{proof}

Our expected regret bound for \banditelfRH{} is below; we defer the remainder of the analysis to Section~\ref{sec:regret-analysis-bandit}.

\begin{theorem} \label{thm:bandit-regret}
Assume that each expert's belief distribution satisfies belief independence and that $T \ge N$. 
Then taking $\varepsilon = (N/T)^{1/3}$, the expected regret of \banditelfRH{} is bounded as
\begin{align*}
\E \left[ \sum_{t=1}^T \loss_{I_t,t} - \min_{j \in [N]} \sum_{t=1}^T \loss_{j,t} \right] 
= O \left( T^{2/3} N^{1/3} \log T\right) .
\end{align*}
\end{theorem}

In the classic, adversarial multi-armed bandit problem --- where incentive compatibility is of no concern --- the minimax expected regret for oblivious adversaries is of order $\Theta(\sqrt{T N})$. In contrast, \banditelfRH{} pays $T^{2/3} N^{1/3} \sqrt{\log T}$ in terms of $T$. This price of the extra multiplicative $\tilde{\Theta}((T/N)^{1/6})$ factor is because the mechanism is exploration separated. It is unclear if this price is avoidable here. We discuss this further in Section~\ref{sec:discussion}.

\section{Regret Analysis} \label{sec:regret-analysis}

\subsection{A general regret analysis for candidate-based FPL-algorithms} \label{sec:general-fpl-analysis}

We begin by introducing a general algorithm and analysis framework that will allow us to analyze the expected regret of our algorithms for both the full-information setting and the bandit feedback setting.

For all $j \in [N]$ and $t \in [T]$, let $\loss_{j,t} \in [0, 1]$ be the loss of expert $j$ in round $t$. For convenience, we also introduce a fictional round 0 where each expert $j$ suffers zero loss $\loss_{j,0} = 0$. Let $X_t = (X_{t,j})_{j \in [N]}$ be a noise random vector satisfying $\E [ X_{j,t} \mid \cF_{t-1} ] = \mathbf{0}$, where $\cF_{t-1}$ is the sigma algebra generated from the history from rounds $1$ through $t-1$. 
Note that for distinct $i, j \in [N]$, we do \emph{not} assume that the noise random variables $X_{i,t}$ and $X_{j,t}$ are conditionally independent given $\cF_{t-1}$. 

We define the perturbed loss $\tilde{\loss}_{j,t} := \loss_{j,t} + X_{j,t}$. 
In round $t$, FPL selects any expert $I_t$ satisfying
\begin{align*}
I_t \in \argmin_{j \in [N]} \sum_{s=0}^{t-1} \tilde{\loss}_{j,s} = \argmin_{j \in [N]} \sum_{s=0}^{t-1} \Woe_{j,s}.
\end{align*}

As already shown in Section~\ref{sec:full-info-regret} and Section~\ref{sec:bandit-regret}, the FPL formulation above generalizes \elfFH{} (for the full-information setting) and \banditelfFH{} (for the bandit setting). 
Note that the probability of exploration is $\varepsilon$ and \banditelfFH{} does not update itself during exploitation rounds, which is reflected in the scaling factor of $1/\varepsilon$. A key 
observation
is that if we set $\varepsilon = 1$, \banditelfFH{} is exactly the same as \elfFH{}. Therefore, we can unify the analysis of these two algorithms. 
Major notation for the analysis can be found in Appendix~\ref{app:notation}.

To analyze FPL's
regret, it will be convenient to introduce a pseudo-algorithm (i.e., an unimplementable algorithm) called Be the Perturbed Leader (BPL), which selects expert $I_{t+1}$ in round $t$.

\begin{lemma} \label{lemma:bpl}
For all $j \in [N]$ and $t \in [T]$, assume that $\tilde{\loss}_{j,t} \in [\bmin, \bmax]$, and define the perturbed loss diameter $D_{\tilde{\loss}} := \bmax - \bmin$. 
Assume for all $j \in [N]$ that $|\tilde{\loss}_{j,0}| \leq \bzero$. Then
\begin{align*}
\E \left[ \sum_{t=1}^T \loss_{I_t,t} - \min_{j \in [N]} \sum_{t=1}^T \loss_{j,t} \right] 
\leq D_{\tilde{\loss}} \sum_{t=1}^T \Pr \left( I_{t+1} \neq I_t \right) + \E \left[ \max_{j \in [N]} \sum_{t=1}^T X_{j,t} \right] + 3 \bzero .
\end{align*}
\end{lemma}
\begin{proof}
For any $j \in [N]$, we have from Lemma 3.1 of \citet{cesa2006prediction} (who attribute it to \citet{hannan1957approximation}; see also the work of \citet{kalai2005efficient}) that
\begin{align*}
\sum_{t=0}^T \tilde{\loss}_{I_{t+1},t} \leq \sum_{t=0}^T \tilde{\loss}_{j,t} .
\end{align*}
Therefore, 
\begin{align*}
\sum_{t=0}^T \tilde{\loss}_{I_t,t} - \sum_{t=0}^T \tilde{\loss}_{j,t} 
\leq \sum_{t=0}^T \tilde{\loss}_{I_t,t} - \sum_{t=0}^T \tilde{\loss}_{I_{t+1},t} 
\leq D_{\tilde{\loss}} \sum_{t=1}^T \ind{I_{t+1} \neq I_t} + 2 \bzero .
\end{align*}
Unpacking notation and recalling that $\loss_{j,0} = 0$ for all $j \in [N]$ gives
\begin{align*}
\sum_{t=1}^T \loss_{I_t,t} - \sum_{t=1}^T \loss_{j,t} 
\leq D_{\tilde{\loss}} \sum_{t=1}^T \ind{I_{t+1} \neq I_t} + 2 \bzero + X_{j,0} + \sum_{t=1}^T X_{j,t} - \sum_{t=0}^T X_{I_t,t} .
\end{align*}
Finally, taking the maximum over $j$ on both sides and then taking the expectation gives
\begin{align*}
\E \left[ \sum_{t=1}^T \loss_{I_t,t} - \min_{j \in [N]} \sum_{t=1}^T \loss_{j,t} \right] 
\leq D_{\tilde{\loss}} \sum_{t=1}^T \Pr \left( I_{t+1} \neq I_t \right) + \E \left[ \max_{j \in [N]} \sum_{t=1}^T X_{j,t} \right] + 3 \bzero ,
\end{align*}
where we used the fact that $\E \left[ X_{I_t,t} \mid \cF_{t-1} \right] = 0$ since $I_t$ is $\cF_{t-1}$-measurable.
\end{proof}

In the above lemma, the event $[I_{t+1} \neq I_t]$ is referred to as a ``leader change''. In order to bound the expected number of leader changes, we adopt the notion of the \emph{lead pack} from \citet{devroye2013prediction}. In round $t$, the lead pack $A_t$ is the set of experts that potentially can ``take the lead'' in round $t+1$, i.e., 
become 
the perturbed leader $I_{t+1}$. Formally, we have
\begin{align*}
A_t = \left\{ j \in [N] \colon \sum_{s=0}^{t-1} \Woe_{j,s} < \min_{k \in [N]} \sum_{s=0}^{t-1} \Woe_{k,s} + 1 \right\} .
\end{align*}

A standard analysis \citep{devroye2013prediction,vanerven2014follow} would bound the probability of a leader change by the probability that the lead pack is of size greater than one. In previous works, such a bound would suffice, but a unique aspect of our algorithms means that we need to go beyond this standard bound. Specifically, this aspect of our algorithms is that they are ``candidate-based''.

\begin{definition}[Candidate-based] \label{def:candidate}
We say that an FPL algorithm is \emph{candidate-based} if there exists a lower bound $\bmin \in \reals$ such that, in each round $t$:
\begin{itemize}
\item using independent randomness, a single expert $C_t$ is drawn from $[N] \cup \{0\}$;
\item $\Pr(C_t = j) = \Pr(C_t = k)$ for all distinct $j, k \in [N]$;
\item $C_t$ satisfies $\tilde{\loss}_{C_t,t} \geq \bmin$, and all $j \in [N] \setminus \{C_t\}$ satisfy $\tilde{\loss}_{j,t} = \bmin$.
\end{itemize}
\end{definition}

The next lemma uses a simple argument to greatly improve upon the standard bound.
\begin{lemma} \label{lemma:leader-change-to-lead-pack}
Take the setup of the previous lemma. In addition, assume that the FPL algorithm is candidate-based with lower bound $\bmin$ (Definition~\ref{def:candidate}). Then for all $t \in [T]$,
\begin{align*}
\Pr(I_{t+1} \neq I_t) \leq \Pr \left( C_t = I_t \right) \cdot \Pr \left( |A_t| > 1 \right) .
\end{align*}
\end{lemma}
\begin{proof}
Since there can be a leader change only if the lead pack contains at least two experts, we have
\begin{align*}
\Pr \left( I_{t+1} \neq I_t \right) 
= \Pr \left( |A_t| > 1 \right) \cdot \Pr \left( I_{t+1} \neq I_t \mid |A_t| > 1 \right) .
\end{align*}
Next, from the nonnegativity of the perturbed losses, observe that the only way expert $I_t$ can be the perturbed leader in round $t$ but fail to be the perturbed leader in round $t+1$ (so that $I_{t+1} \neq I_t$) is if $\tilde{\loss}_{I_t,t} > \bmin$. The latter can happen only if $I_t$ is selected to be the candidate in round $t$ (i.e., $C_t = I_t$). Therefore, 
\begin{align*}
\Pr \left( I_{t+1} \neq I_t \mid |A_t| > 1 \right) 
\leq \Pr \left( C_t = I_t \mid |A_t| > 1 \right) 
= \Pr \left( C_t = I_t \right) ,
\end{align*}
where the second inequality is because $C_t$ depends on independent randomness and $C_t$ takes all values in $[N]$ with equal probability.
\end{proof}

The use of a sad lottery is vital to the proof of the previous lemma. Briefly, in a sad lottery, the leader is guaranteed to maintain its lead whenever it is \emph{not} selected as the candidate; this non-selection event happens with probability $\Omega(1 - \frac{1}{N})$. On the other hand, in a happy lottery, the leader is only guaranteed to maintain its lead whenever it \emph{is} selected as the candidate (so that no other expert could have been selected as the candidate), and this selection event happens with probability $O(\frac{1}{N})$. As a result, it seems that the use of a sad lottery reduces the number of leader changes and hences leads to an algorithm that is more stable.

Before applying Lemma~\ref{lemma:leader-change-to-lead-pack} in Lemma~\ref{lemma:bpl}, we present a technical lemma to control the second term in Lemma~\ref{lemma:bpl}; the proof uses more or less standard ideas from empirical process theory.

\begin{lemma} \label{lemma:E-max-abs} 
Let $q$ be a variable that can depend only on $N$ and $T$. 
Let $(X_{j,s})_{j \in [N], s \in [t]}$ be centered random variables satisfying $|X_{j,s}| \leq \frac{\babs}{q}$ for positive value $\babs$, with joint law satisfying the following properties:
\begin{itemize}
\item $X_1, \ldots, X_t$ are independent;
\item for each $(j, s)$, with marginal probability at least $1 - q$, it holds that $|X_{j,s}| \leq \babs$;
\end{itemize}
Assume that $t \geq \frac{3}{q} \log \left( \frac{N}{\sqrt{q}} \right)$. 
Then
\begin{align*}
\E \left[ \max_{j \in [N]} \left| \sum_{s=1}^t X_{j,s} \right| \right] \leq 4 \babs \sqrt{\frac{t \log (2 N)}{q}} .
\end{align*}
\end{lemma}

Let us give an idea of how this result will be used in our regret analysis. The variable $q$ corresponds to the probability that a given expert gets a sad point in a given round (i.e., $W_{j,s} = 1$), which is at most $1 / N$ in the full-information setting and at most $\varepsilon / N$ in the bandit setting. When an expert gets a sad point, the corresponding noise $X_{j,s}$ can be at a scale of $\babs / q$ due to importance weighting, and in the more likely event that $W_{j,s} = 0$, the noise magnitude is at most $\babs$; here, $\babs$ may be thought of as a moderate constant (in all our applications, $\babs$ is at most 4).

Applying Lemmas~\ref{lemma:leader-change-to-lead-pack}~and~\ref{lemma:E-max-abs} in Lemma~\ref{lemma:bpl} immediately gives the following general regret bound that is suitable for both our full-information and bandit algorithms.

\begin{corollary} \label{cor:general-pre-regret-bound}
Take the setup of Lemma~\ref{lemma:bpl}. In addition, assume that the FPL algorithm is candidate-based with lower bound $\bmin$ (Definition~\ref{def:candidate}). Further assume that $(X_{j,s})_{j \in [N], s \in [t]}$ satisfy the conditions given in Lemma~\ref{lemma:E-max-abs} with constants $\babs$ and $q$. Then if $T \geq \frac{3}{q} \log \left( \frac{N}{\sqrt{q}} \right)$,
\begin{align*}
\E \left[ \sum_{t=1}^T \loss_{I_t,t} - \min_{j \in [N]} \sum_{t=1}^T \loss_{j,t} \right] 
\leq D_{\tilde{\loss}} \sum_{t=1}^T \Pr( C_t = I_t ) \cdot \Pr \left( |A_t| > 1 \right) 
       + 4 \babs \sqrt{\frac{T \log N}{q}} 
       + 3 \bzero.
\end{align*}
\end{corollary}

This regret bound is incomplete until we bound the two terms in the summation on the right-hand side. The term $\Pr(C_t = I_t)$ will be easy to control. The term $\Pr(|A_t| > 1)$ is the most challenging to analyze and is what we focus on next.

\subsection{Analyzing the lead pack} \label{sec:lead-pack}

In this section, we analyze the probability of having a lead pack size larger than one by boiling it down to Poisson binomial large deviations. To handle both the full-information setting and the bandit setting, we introduce $\lPoi$, which represents the minimum probability such that an expert gets a sad point. In the case of the full-information setting, $\lPoi = 1/(2N)$. In the case of the bandit setting, $\lPoi = \varepsilon/(2N)$ since the expert only gets a sad point during the exploration rounds (which occur at a rate of $\varepsilon$). By setting $\varepsilon = 1$ in the full-information setting, we can unify the analysis for the full-information and bandit cases.
Since $N \ge 2$, we have $\lPoi \le 1/4$. For both settings, we have
\[
\Pr[W_{i,t} = 1] \in [\lPoi, 2\lPoi].
\]

\begin{theorem}[Lead Pack] \label{thm:lead-pack-prob-bound}
Assume that 
\begin{equation}\label{ineq_tmin}
24 + \sqrt{8 \lPoi t \log (Nt/\lPoi)} < \lPoi t/24 .
\end{equation}
Then, we have
\[
\Pr(|A_t| > 1 \mid \Woe_0 = \woe_0)
\le 
\frac{C_1 \log t}{\sqrt{\lPoi t}}
+
C_2 \sqrt{\frac{t \log N}{\lPoi t}}
\]
for some universal constants $C_1, C_2 > 0$ and for any initial noise $\woe_0 = (w_{1,0},\woe_{2,0},\dots,\woe_{N,0})$.
\end{theorem}
Theorem \ref{thm:lead-pack-prob-bound} immediately yields the marginalized probability 
\[
\Pr(|A_t| > 1) = \E\left[ \Pr(|A_t| > 1 | \Woe_0 = \woe_0) \right]
\le 
\frac{C_1 \log t}{\sqrt{\lPoi t}}
+
C_2 \sqrt{\frac{t \log N}{\lPoi t}},
\]
which we use in Section \ref{sec:regret-analysis-full-info} and Section \ref{sec:regret-analysis-bandit}.

\begin{proof}[Proof sketch of Theorem \ref{thm:lead-pack-prob-bound}] 
Due to page constraints, this section only describes the core idea of the proof.
The complete proof is shown in Appendix \ref{subsec:lead-pack-prob-bound}.

In the proof, we show an equality 
\[
\sum_{s=0}^{t-1} W_{i,s} = 
\frac{\varepsilon}{4 N} \sum_{s=1}^{t-1} \loss_{i,s-1} + W_{0,i} + Z_{i,t-1} + \mathrm{Const}
\]
where $\frac{\varepsilon}{4 N} \sum_{s=1}^{t-1} \loss_{i,s-1}$ is a normalized total loss, $W_{0,i}$ is an initial noise random variable, 
and $Z_{i,t-1}$ is a Poisson binomial random variable (i.e., a sum of Bernoulli random variables with inhomogeneous success probabilities).
Since we consider an oblivious adversary, we can assume that the loss matrix $(l_{i,t})_{i \in [N], t \in [T]}$ is fixed. 
Moreover, we consider a conditional probability given the initial tie-breaking noise $W_0 = w_0$. Therefore, 
\[
L_{i,t-1} := \sum_{s=1}^{t-1} \loss_{i,s-1} + \frac{4 N}{\varepsilon} w_{i,0},
\]
can be viewed as a constant, and 
\[
I_t := \argmin_{i \in [N]} \sum_{s=1}^{t-1} W_{i,s} 
= \argmin_{i \in [N]} \left\{ \frac{\varepsilon}{4 N} L_{i,t-1} + Z_{i,t-1} \right\} .
\]
Our goal is to bound the probability of having a non-single lead pack (i.e., $\Pr[|A_t| > 1]$) against any constant matrix $(L_{i,t})_{i \in [N], t \in [T]}$.

Overall, the proof follows similar steps as ``B. Bounding the number of switches'' in \citep{devroye2013prediction} for a symmetric binomial random variable.\footnote{\citet{devroye2013prediction} considered the case that each noise is $\pm 1$ with a fifty-fifty probability.} 
As we generalize it to Poisson binomial random variables, our proof is much more involved than theirs. In particular, unlike symmetric binomial random variable, a closed-form formula on the tail ratio is unavailable, our results rely on lower bounds, and we need more careful discussions related to this fact.

To conclude this sketch, let us describe the core idea in bounding the Poisson binomial tail. 
We reduce the problem of getting a tail bound for a Poisson binomial distribution to that of getting a tail bound for a binomial distribution 
by using the following \textit{separation lemma}, which states that moving two of the parameters in opposite directions reduces the tail probability. Repeating this operation homogenizes the parameters; all but one parameter are $\lPoi$ or $\uPoi$. If we set the ceiling to $\uPoi=1$, we obtain a binomial distribution with parameter $\lPoi$ and a set of deterministic values of ones.
\begin{figure}[h]
\hspace{-16em}\includegraphics[width=1\textwidth]{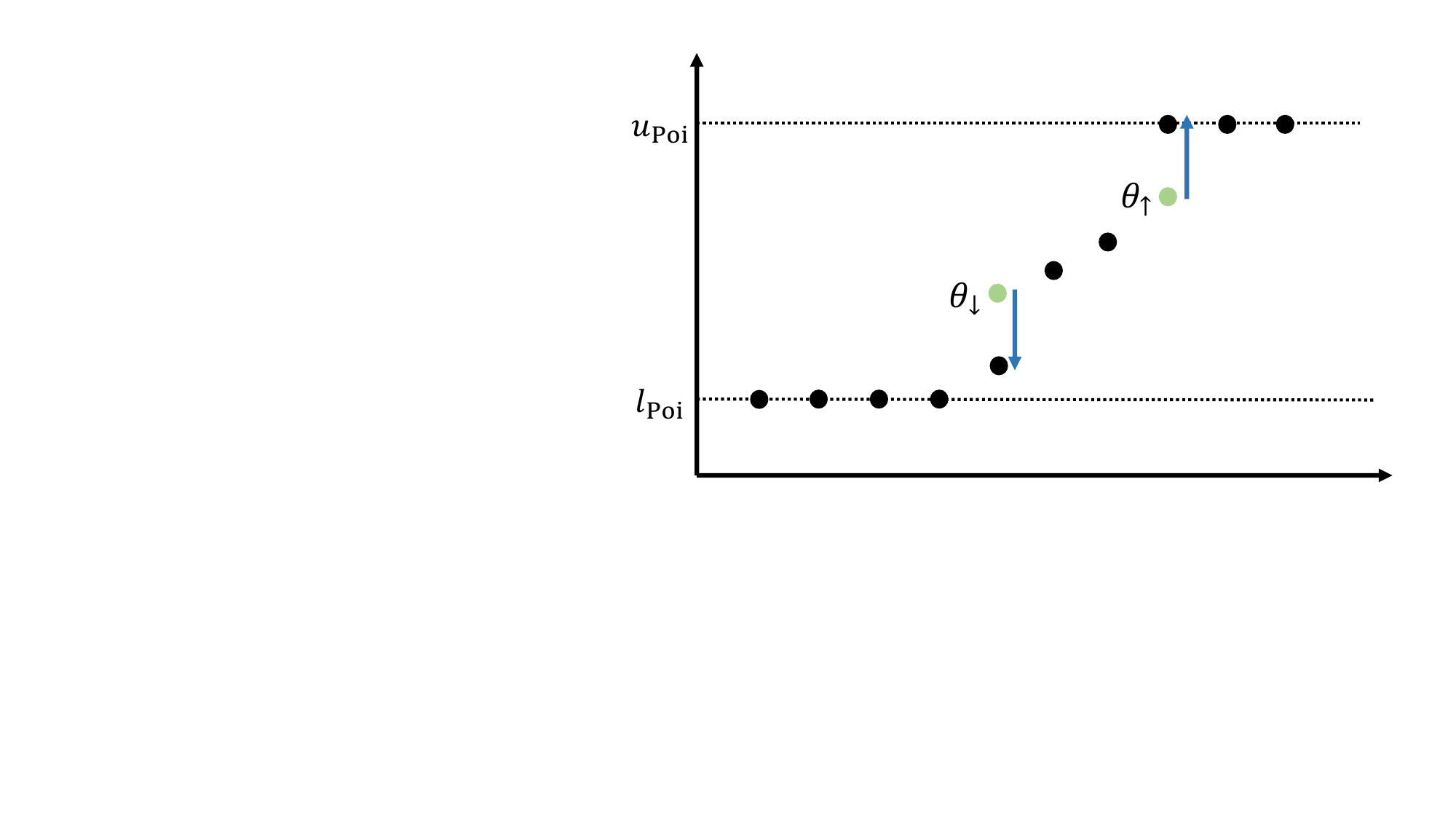}
\vspace{-8em}
  \caption{Illustration of the separation lemma, which is formalized in Lemma \ref{lem_operation} in the Appendix~\ref{sec_poisson}. Black dots represent the Poisson binomial parameters in non-decreasing order. We move two parameters $\theta_{\downarrow},\theta_{\uparrow}$ for the same distance until one of them hits the floor ($\lPoi$) or the ceiling ($\uPoi$).}
  \label{fig:separation_main}
\end{figure}
\end{proof} %Sketch: thm:lead-pack-prob-bound

\subsection{Applying general results for full-information setting}  \label{sec:regret-analysis-full-info}

Everything is in place to analyze the expected regret of \elfFH{}. For all $j \in [N]$, set $\tilde{\loss}_{j,t} = 4 N \cdot \Woe_{j,t} - 2$ when $t \in [T]$ and $\tilde{\loss}_{j,0} = 4 N \cdot \Woe_{j,0}$. Then we have $\bmin = -2$ and $\bmax = 4 N - 2$, so that $D_{\tilde{\loss}} = 4 N$; also, $\bzero = 1$. In addition, we can set $q$ and $\babs$ (originally from Lemma~\ref{lemma:E-max-abs}) as $q = \frac{1}{N}$ and $\babs = 4$. The next corollary is immediate from Corollary~\ref{cor:general-pre-regret-bound}.

\begin{corollary} \label{cor:full-info-pre-bound}
The expected regret of \elfFH{} is bounded as
\begin{align*}
\E \left[ \sum_{t=1}^T \loss_{I_t,t} - \min_{j \in [N]} \sum_{t=1}^T \loss_{j,t} \right] 
\leq 4 \sum_{t=1}^T \Pr \left( |A_t| > 1 \right) 
       + 16 \sqrt{T N \log N} 
       + 3 .
\end{align*}
\end{corollary}

\begin{proof}
Apply Corollary~\ref{cor:general-pre-regret-bound} with $D_{\tilde{\loss}} = 4 N$, use Lemma~\ref{lemma:leader-change-to-lead-pack} to bound $\Pr \left( I_{t+1} \neq I_t\right)$ by $\Pr(C_t = I_t) \cdot \Pr(|A_t| > 1)$, and use the fact that $\Pr(C_t = I_t) = \frac{1}{N}$.
\end{proof}

Next, we upper bound the term $\sum_{t=1}^T \Pr \left( |A_t| > 1 \right)$ in Corollary~\ref{cor:full-info-pre-bound} by applying Theorem~\ref{thm:lead-pack-prob-bound} with the choice of $\lPoi = 1/(2N)$ to get
\begin{align*}
\sum_{t=1}^T \Pr \left( |A_t| > 1 \right) 
&= \sum_{t=1}^T \mathbf{1}[24 + 4 \sqrt{ (t/N) \log(Nt/\lPoi)} > t/(48N)] 
+
\sum_{t=1}^T O\left((\log t+\sqrt{\log N})\sqrt{\frac{N}{t}}\right)\\
&= O(N\log N) + O(\sqrt{T N} \log T)\\
&= O(\sqrt{T N} \log T ) \tag{by $T \ge N$} .
\end{align*}
Using this result in Corollary~\ref{cor:full-info-pre-bound} yields Theorem~\ref{thm:full-info-regret}.

\subsection{Applying general results for bandit setting} \label{sec:regret-analysis-bandit}

Now, let us bound the expected regret of \banditelfFH{}. For all $j \in [N]$, set $\tilde{\loss}_{j,t} = \frac{4 N}{\varepsilon} \cdot \Woe_{j,t} - 2$ when $t \in [T]$ and $\tilde{\loss}_{j,0} = \frac{4 N}{\varepsilon} \Woe_{j,0}$. Then we have $\bmin = -2$ and $\bmax = \frac{4 N}{\varepsilon} - 2$, so that $D_{\tilde{\loss}} = \frac{4 N}{\varepsilon}$; also, $\bzero = 1$. Finally, we can set $q$ and $\babs$ (originally from Lemma~\ref{lemma:E-max-abs}) as $q = \frac{\varepsilon}{N}$ and $\babs = 4$. The next corollary is analogous to our corollary for the full-information setting.

\begin{corollary} \label{cor:bandit-pre-bound}
The expected regret of \banditelfFH{} is bounded as
\begin{align*}
\E^\fh \left[ \sum_{t=1}^T \loss_{I_t,t} - \min_{j \in [N]} \sum_{t=1}^T \loss_{j,t} \right] 
\leq 4 \sum_{t=1}^T \Pr^\fh \left( |A_t| > 1 \right) 
       + 16 \sqrt{\frac{T N \log N}{\varepsilon}} 
       + 3 .
\end{align*}
\end{corollary}

\begin{proof}
Apply Corollary~\ref{cor:general-pre-regret-bound} with $D_{\tilde{\loss}} = \frac{4 N}{\varepsilon}$, use Lemma~\ref{lemma:leader-change-to-lead-pack} to bound $\Pr \left( I_{t+1} \neq I_t\right)$ by $\Pr(C_t = I_t) \cdot \Pr(|A_t| > 1)$, and use the fact that $\Pr(C_t = I_t) = \frac{\varepsilon}{N}$.
\end{proof}

To bound the expected regret of \banditelfRH{}, we plug the above regret bound into Proposition~\ref{prop:cond-rh-to-fh} with the choice $\varepsilon = (N/T)^{1/3}$. Applying Theorem~\ref{thm:lead-pack-prob-bound} with $\lPoi=\varepsilon/(2N)$ gives
\begin{align}
\sum_{t=1}^T \Pr \left( |A_t| > 1 \right) 
&= \sum_{t=1}^T \mathbf{1}[24 + 4 \sqrt{  (\varepsilon t/N) \log(Nt/\lPoi)} > \varepsilon t/(48N)] 
+
\sum_{t=1}^T O\left((\log t+\sqrt{\log N}) \sqrt{\frac{N}{\varepsilon t}}\right) \nonumber\\
&= O(T^{1/3} N^{2/3} (\log T+\sqrt{\log N}) + O(T^{2/3} N^{1/3} (\log T+\sqrt{\log N})) \nonumber\\
&=O(T^{2/3} N^{1/3} (\log T)) \tag{by $T \ge N$}
,
\end{align}
after which Theorem~\ref{thm:bandit-regret} follows.

\section{Discussion} \label{sec:discussion}

This work proposed the first no-regret mechanisms for the online forecasting competition setting that are \truthful{} for non-myopic experts. Our algorithms are based on the Independent-Event Lotteries Forecasting Competition Mechanism (I-ELF). 
To derive regret bounds, we show that I-ELF can be viewed as an instance of Follow the Perturbed Leader with random-walk perturbation. 
For our versions of FPL, the noise level is heterogeneous over the rounds, and thus a na\"ive, direct application of existing results \citep{devroye2013prediction} does not work. 
To cope with such heterogeneity, we introduced upper and lower bounds for Poisson binomial distributions, which are of independent interest.

\paragraph{Conjectured lower bound.}

Consider the full information setting. We conjecture that all Online IC-BI algorithms (informally, all \truthful{} algorithms) must have belief regret lower bounded as $\Omega(\sqrt{N T})$. We have not yet succeeded in proving this conjecture. Here, we briefly mention some approaches that may at first seem promising towards getting better regret but ultimately fail to give algorithms that are \truthful{}. 
One approach is a version of \elfRH{} that allows multiple lottery winners (i.e., multiple winners of sad points) per round. The next theorem states that this mechanism is not Online IC-BI; the proof is in Appendix~\ref{app:multiple-draws-not-ic}.
\begin{theorem} \label{thm:multiple-draws-not-ic}
Consider a modified version of \elfRH{} in which, for each round $t$, each past round $s \in [t-1]$, and each expert $j \in [N]$, the random variable $W_{j,s}$ is \emph{independently} drawn from $\mathrm{Bernoulli}(\frac{1}{2} + \frac{1}{4} \loss_{j,s})$. 
Let $N = 2$ and $T = 2$. Then for any expert $i$, there exists a belief distribution such that this modified version of \elfRH{} is not Online IC-BI.
\end{theorem}

\paragraph{Multiple outcomes.}
As already mentioned by \cite{witkowski2023incentive}, an advantage of ELF-based algorithms (including our own) is that they automatically extend to the case of multiple outcomes, i.e., categorical outcomes. The only necessary restriction is that the strictly proper loss has range in $[0, 1]$. In contrast, for the previous work of  \citet{frongillo2021efficient} which also considers nonmyopic experts (although under approximate truthfulness with experts who have immutable beliefs), it is not immediately clear how to extend their results to case of multiple outcomes. Technically, to handle the Brier loss over $m$ outcomes, we would need to first rescale the loss --- whose native scale is $[0, m]$ --- and undo this scaling to state our regret bounds in terms of the natively scaled Brier loss. This scaling introduces an extra factor of $m$. It would be interesting to see if this factor can be reduced.

\paragraph{Bandit feedback with one extra observation.}
Consider the following augmented bandit protocol: in addition to selecting an expert in each round and using that expert's report (and suffering that expert's loss), the algorithm also uniformly at random selects one \emph{additional} expert. The algorithm does not suffer the loss of this additional expert, but it does observe this additional expert's report (and hence its loss). This setting fits into a  protocol studied by \citet{avner2012decoupling}, which they generally refer to as decoupling exploration and exploitation. Pragmatically, if we consider each expert as providing a report only if it is paid some fixed sum of money, then the mechanism only needs to double its payment per round compared to the bandit setting. Technically, it is not hard to see that our full-information algorithm \elfRH{} --- with one small modification --- already achieves regret $O(\sqrt{T N} \log T)$. 
Specifically, referring to \elfRH{}, the selected expert in each round is $I_t$, and the additionally selected expert is $C_t$. The modification compared to \elfRH{} is that once a candidate $C_t$ is drawn, this choice is never revised. Even so, the lottery draws $W_{j,t}$ (from $\mathrm{Bernoulli}\left( \frac{1}{2} + \frac{1}{4} \loss_{C_t,t} \right)$) are done anew each round, which is enough to preserve incentive compatibility. 

\paragraph{The necessity of the assumption of belief independence.} 

The assumption is necessary in several ways. For simplicity, suppose that $N = 2$. We consider the strategy of one of the experts, expert $i$, and refer to the other expert as expert $k$. If we do not assume belief independence, then there can be serial correlation among the losses $\loss(r_{k,1}, o_1), \loss(r_{k,2}, o_2), \ldots, \loss(r_{k,T}, o_T)$. Concretely, let us suppose that expert $k$'s loss in round 1 is either large or small. Next, according to expert $i$'s belief, it could happen that conditional on expert $k$'s loss being large in round 1, expert $k$'s loss in all future rounds is large. Hence, expert $k$ can be competitive only when expert $k$'s loss in the first round is large. With this setup, it can be advantageous for expert $i$ to obtain low loss in round $1$ conditional on expert $k$'s loss being large, i.e., to maximize its performance only when the other expert will be competitive. Technically, in round 1, expert $i$'s optimal strategy is to select the report that minimizes its (subjective) conditional expected loss, where the conditioning is on expert $k$'s loss being large in round 1. Hence, \truthfulness{} is violated.

\subsubsection*{Acknowledgements}

Nishant Mehta and Ali Mortazavi were supported by the NSERC Discovery Grant RGPIN-2018-03942. 
Much of this work happened while Nishant Mehta was visiting Junpei Komiyama at NYU Stern in the summer of 2024. The authors thank NYU Stern for hosting this visit.

\newpage

\bibliography{nonmyopic_ic_online_learning}

\newpage

\appendix

\section{Summary of notation for regret analysis}
\label{app:notation}

Let $\varepsilon = 1$ for the full-information case. Section \ref{sec:regret-analysis} provides a unified regret analysis that is applicable to both the full-information and bandit algorithms.
The relationship between the parameters is as follows. Note that $t=0$ is a hypothetical round that incurs a small continuous noise for tiebreaking with probability $1$.
\begin{align}
\loss_{i,t} &
  \in [0,1] \text{\ \ \ (the \textit{loss} of expert $i$ at round $t$, where $\loss_{i,0} = 0$)} \\
W_{i,t} &\sim 
\begin{cases}
  \mathrm{Bernoulli}\left(\frac{\varepsilon}{N}\left(\frac{1}{2}+\frac{\loss_{i,t}}{4}\right)\right) \in \{0,1\} & \text{if } t > 0 \\
  \mathrm{Uniform}([-\frac{\varepsilon}{4 N}, \frac{\varepsilon}{4 N}])  & \text{if } t = 0.
\end{cases} 
\text{\ \ \ (the \textit{woe} of expert $i$ at round $t$)} 
\\
\tilde{\loss}_{i,t} &= 
\begin{cases}
  \left(\frac{4N}{\varepsilon}W_{i,t}\right) - 2 & \text{if } t > 0 \\
  \frac{4N}{\varepsilon}W_{i,t}   & \text{if } t = 0.
\end{cases} \\
X_{i,t} &= \tilde{\loss}_{i,t} - \loss_{i,t} = 
\begin{cases}
  \frac{4N}{\varepsilon}W_{i,t} - 2 - \loss_{i,t} \in \left[-3, \frac{4N}{\varepsilon}\right] & \text{if } t > 0, \\
  \frac{4N}{\varepsilon}W_{i,t}  & \text{if } t = 0.
\end{cases} 
\text{\ \ \ (unbiased noise, $\mathbb{E}[X_{i,t}] = 0$)}
\label{ineq_xdiscuss}\\
Z_{i,t} &= \frac{\varepsilon}{4N}\sum_{s=1}^t X_{s,t}
\sim \mathrm{PBIN}\left( \left(\frac{\varepsilon}{4N}\left(\frac{1}{2} + \frac{l_{i,s}}{4}\right)\right)_{s \in [t]} \right) - \E\left[\sum_{s=1}^t \frac{\varepsilon}{4N}\left(\frac{1}{2} + \frac{l_{i,s}}{4}\right)\right]\nonumber\\
&\text{\ \ \ \ \ \ \ \ \ (centered Poisson binomial)}\\
I_t &= \underbrace{\argmin_i \sum_{s=0}^{t-1} \left(l_{i,s} + X_{i,s}\right)}_{\text{(FPL notation)}}
 = \underbrace{\argmin_i \sum_{s=0}^{t-1} \Woe_{i,s}}_{\text{(sad lottery notation)}}
\end{align}

\section{Incentive compatibility proofs} \label{app:ic}

\subsection{Full-information setting}

As mentioned in the main text, we actually will prove that a more general class of online mechanisms is Online IC-BI. This more general class, which we call \gelfRH{}, is the same as \elfRH{} (Algorithm~\ref{alg:simpelf-rh}) except for the lottery; to be precise, \gelfRH{} has more freedom in the choice of distribution of $\Woe_{j,s}$ conditional on $C_s$. \gelfRH{} has additional parameters $a_1$, $a_2$, and $\rho$ with joint range
\begin{align}
\begin{array}{c}
\displaystyle
0 \leq a_1 \leq 1 ; \\
\displaystyle
0 < a_2 \leq 1 ; \\
\displaystyle
1 - \frac{1 - a_1}{a_2} \leq \rho \leq \frac{a_1}{a_2} .
\end{array} \label{eqn:woe-joint-range}
\end{align}
We will see the reason for this joint parameter range shortly. When using \gelfRH{}, in any given round, at most one expert wins the lottery; for certain settings of the parameters, it happens with positive probability that no expert wins the given round's lottery. 
To ease the analysis, we introduce expert 0 as a dummy expert. In any round where none of experts 1 through $N$ win the lottery, expert 0 wins the lottery. Let $W_t$ be a categorical random variable indicating the winner of the lottery in round $t$. In \gelfRH{}, we have
\begin{align}
\Pr(W_t = i) =
\begin{cases}
\displaystyle 
\frac{1}{N} \left( a_1 + a_2 \loss_{i,t} - \frac{a_2 \rho}{N-1} \sum_{j \in [N] \setminus \{i\}} \loss_{j,t} \right) & \text{if } i \in [N] \\
\displaystyle 
1 - a_1 - \frac{a_2 (1 - \rho)}{N} \sum_{j=1}^N \loss_{j,t} & \text{if } i = 0 \\
\end{cases}
\label{eqn:woe-generalized-elf}
\end{align}
To better understand this generalization, let us first consider the case of $a_1 = 1$ and $a_2 = 1$. With these settings, observe that $\rho$ represents the amount of probability that is redistributed across the experts. Setting $\rho = 1$ recovers a sad lottery analogue of (the online extension of) I-ELF. We refer to the class of mechanisms captured by $\rho = 0$ and any valid setting of $a_1$ and $a_2$ as a \emph{variant} of \elfRH{}. We reserve the term \elfRH{} for the particular variant induced by taking $a_1 = \frac{1}{2}$ and $a_2 = \frac{1}{4}$.

When \citet{witkowski2023incentive} proved that I-ELF is IC-BI, they relied on the assumption that the score is never equal to 1. The reason for this assumption is to ensure that each expert has a positive probability of getting an internal win in each round; having this probability be positive is vital for the proof of incentive compatibility. Note that if $a_1 > a_2 \rho$ (as with \elfRH{}), this assumption is unnecessary. For a standard score in the range of $[0, 1]$, a simple trick to enforce this assumption is to rescale the score by some constant $\beta$ that is less than but arbitrarily close to 1. Exactly the same trick works for losses in the range of $[0, 1]$. So, whenever $a_1 = a_2 \rho$, we implicitly will assume that the loss function has been rescaled by some constant $\beta$ that is less than but arbitrarily close to 1, rendering losses in the range $[0, \beta] \subsetneq [0, 1]$. Again, we emphasize that this rescaling \emph{is not necessary} for \elfRH{} nor any other setting of the parameters for which $a_1 > a_2 \rho$.

\paragraph{Explanation of joint range of parameters.}

Next, we briefly explain the joint parameter range for $a_1$, $a_2$, and $\rho$. From the case of $\Pr(W_t = 0)$, by considering possible loss values it is clear that we must ensure that $1 - a_1 \in [0, 1]$ and $1 - a_1 - a_2 (1 - \rho) \in [0, 1]$. From the case of $\Pr(W_t = i)$ for $i \in [N]$, we must ensure that $a_1 - a_2 \rho \geq 0$. Finally, for our proof of incentive compatibility, we will need to impose $a_2 > 0$ (otherwise, an expert will not care to minimize its expected loss). These conditions together imply the joint range given in \eqref{eqn:woe-joint-range}; we note that the condition $a_2 \leq 1$ is needed in order for there to be a non-empty feasible set for $\rho$.

We are now ready to state our incentive compatibility result for \gelfRH{}; this result covers \elfRH{} (Theorem~\ref{thm:full-info-IC}) as a special case.

\begin{theorem} \label{thm:full-info-IC-general}
For any setting of the parameters satisfying \eqref{eqn:woe-joint-range}, \gelfRH{} is Online IC-BI.
\end{theorem}

\begin{proof}
It suffices to show, for an arbitrary ``decision round'' $t \in [T]$ and an arbitrary ``target round'' $T' \in \{t, t + 1, \ldots, T\}$, that any expert $i \in [N]$, having observed $r_{1:t-1}$ and $o_{1:t-1}$, strictly maximizes its subjective probability
\begin{align*}
\Pr \left( \sum_{s=0}^{T'} W_{i,s} < \min_{j \in [N] \setminus \{i\}} \sum_{s=0}^{T'} W_{j,s} \bigmid r_{1:t-1}, o_{1:t-1} \right)
\end{align*}
by truthfully selecting $r_{i,t} = b_{i,t}$. Indeed, if we can show this for arbitrary $T'$ as above, then the expert reports truthfully regardless of the target round, so Online IC-BI follows.

Fix such a pair $(t, T')$ as well as a history $(r_{1:t-1}, o_{1:t-1})$. The rest of the proof always conditions on $r_{1:t-1}, o_{1:t-1}$. To avoid clutter, let $\Pr_t(\cdot)$ denote the conditional expectation $\Pr(\cdot \mid r_{t-1}, o_{t-1})$.

Adopting the notation $[N]_{-i} := [N] \setminus \{i\}$ and letting $\sum_{s \neq t}$ be shorthand for $\sum_{s \in \{0, 1, \ldots, T'\} \setminus \{t\}}$, we rewrite the above probability as
\begin{align}
&\Pr_t \left( \forall j \in [N]_{-i} \colon \sum_{s \neq t} W_{i,s} + W_{i,t} < \sum_{s \neq t} W_{j,s} + W_{j,t} \right) \nonumber \\
&= \Pr_t \left( \forall j \in [N]_{-i} \colon W_{i,t} < \sum_{s \neq t} W_{j,s} + W_{j,t} - \sum_{s \neq t} W_{i,s} \right) \nonumber \\
&= \Pr_t \left( W_{i,t} = 1 , \left[ \forall j \in [N]_{-i} \colon 1 < \sum_{s \neq t} W_{j,s} - \sum_{s \neq t} W_{i,s} \right] \right) \nonumber \\
&\quad+ \sum_{k \in [N]_{-i}} \Pr_t \left( W_{k,t} = 1, \left[ \forall j \in [N]_{-i} \colon 0 < \sum_{s \neq t} W_{j,s} + \ind{j = k} - \sum_{s \neq t} W_{i,s} \right] \right) \nonumber \\
&\quad+ \Pr_t \left( W_{0,t} = 1, \left[ \forall j \in [N]_{-i} \colon 0 < \sum_{s \neq t} W_{j,s} - \sum_{s \neq t} W_{i,s} \right] \right) \nonumber \\
&= \Pr_t \left( W_{i,t} = 1 \right) \cdot \Pr_t \left( \forall j \in [N]_{-i} \colon 1 < \sum_{s \neq t} W_{j,s} - \sum_{s \neq t} W_{i,s} \bigmid W_{i,t} = 1 \right) \label{eqn:woe-pr-term-type-i} \\
&\quad+ \sum_{k \in [N]_{-i}} \Pr_t \left( W_{k,t} = 1 \right) \cdot \Pr_t \left( \forall j \in [N]_{-i} \colon 0 < \sum_{s \neq t} W_{j,s} + \ind{j = k} - \sum_{s \neq t} W_{i,s} \bigmid W_{k,t} = 1 \right) \label{eqn:woe-pr-term-type-k} \\
&\quad+ \Pr_t \left( W_{0,t} = 1 \right) \cdot \Pr_t \left( \forall j \in [N]_{-i} \colon 0 < \sum_{s \neq t} W_{j,s} - \sum_{s \neq t} W_{i,s} \bigmid W_{0,t} = 1 \right) . \label{eqn:woe-pr-term-type-0}
\end{align}

Recall that $\Pr_t$ means we condition on $r_{1:t-1}$ and $o_{t-1}$. Now, from belief independence, it is clear that we can drop this conditioning on all terms of the type $\Pr_t(W_{j,t} = 1)$ in equations \eqref{eqn:woe-pr-term-type-i}, \eqref{eqn:woe-pr-term-type-k}, and \eqref{eqn:woe-pr-term-type-0}. 
In addition, belief independence allows us to drop the conditioning on $W_{j,t} = 1$ in the second probability term in each of equations \eqref{eqn:woe-pr-term-type-i}, \eqref{eqn:woe-pr-term-type-k}, and \eqref{eqn:woe-pr-term-type-0}. To see why, observe that if $Y_1$, $Y_2$, and $Y_3$ are random variables such that each of $Y_1$ and $Y_2$ is independent of $Y_3$. Then we have $\Pr(Y_1 \mid Y_2, Y_3) = \Pr(Y_1 \mid Y_2)$. As an example, we can apply this in \eqref{eqn:woe-pr-term-type-i} by taking $Y_1 = \ind{\forall j \in [N]_{-i} \colon 1 < \sum_{s \neq t} W_{j,s} - \sum_{s \neq t} W_{i,s}}$, $Y_2 = \ind{R_{1:t-1} = r_{1:t-1}, O_{1:t-1} = o_{1:t-1}}$, and $Y_3 = \ind{W_{i,t} = 1}$.

These simplifications yield
\begin{align}
\begin{split}
&\Pr \left( \sum_{s=0}^{T'} W_{i,s} < \min_{j \neq i} \sum_{s=0}^{T'} W_{j,s} \bigmid r_{1:t-1}, o_{1:t-1} \right) \\
&= \Pr \left( W_{i,t} = 1 \right) \cdot \Pr_t \left( \forall j \in [N]_{-i} \colon 1 < \sum_{s \neq t} W_{j,s} - \sum_{s \neq t} W_{i,s} \right) \\
&\quad + \sum_{k \in [N]_{-i}} \Pr \left( W_{k,t} = 1 \right) \cdot \Pr_t \left( \forall j \in [N]_{-i} \colon 0 < \sum_{s \neq t} W_{j,s} + \ind{j = k} - \sum_{s \neq t} W_{i,s} \right) \\
&\quad + \Pr \left( W_{0,t} = 1 \right) \cdot \Pr_t \left( \forall j \in [N]_{-i} \colon 0 < \sum_{s \neq t} W_{j,s} - \sum_{s \neq t} W_{i,s} \right) \\
&= \Pr \left( W_{i,t} = 1 \right) \cdot \underbrace{\Pr_t \left( \forall j \in [N]_{-i} \colon \sum_{s \neq t} W_{j,s} - \sum_{s \neq t} W_{i,s} > 1 \right)}_{c_i} \\
&\quad + \sum_{k \in [N]_{-i}} \Pr \left( W_{k,t} = 1 \right) \cdot \underbrace{\Pr_t \left( \forall j \in [N]_{-i} \colon \sum_{s \neq t} W_{j,s} - \sum_{s \neq t} W_{i,s} > -\ind{j = k} \right)}_{c_k} \\
&\quad + \Pr \left( W_{0,t} = 1 \right) \cdot \underbrace{\Pr_t \left( \forall j \in [N]_{-i} \colon \sum_{s \neq t} W_{j,s} - \sum_{s \neq t} W_{i,s} > 0 \right)}_{c_0} ,
\end{split} \label{eqn:woe-ci-ck-c0}
\end{align}
where the last line is just a natural rearrangement and the notation $c_i$ and $c_k$ is for convenience.

We rewrite the above as
\begin{align}
&c_i \cdot \left( 1 - \sum_{k \in [N]_{-i}} \Pr(W_{k,t} = 1)  - \Pr(W_{0,t} = 1) \right) + \sum_{k \in [N]_{-i}} c_k \Pr(W_{k,t} = 1) + c_0 \Pr(W_{0,t} = 1) \nonumber \\
&= c_i - \sum_{k \in [N]_{-i}} (c_i - c_k) \Pr(W_{k,t} = 1) - (c_i - c_0) \Pr(W_{0,t} = 1) .
\label{eqn:expert-i-incentive-objective}
\end{align}

Recall that the expert's goal is to select a report $r_{i,t}$ that maximizes the above expression. For all $k \in [N]_{-i}$, it clearly holds that $c_i \leq c_0 \leq c_k$. In Proposition~\ref{prop:woe-ci-c0} (stated and proved immediately after this proof), we can and will show something even stronger: it holds that $c_i < c_0$ and hence also $c_i < c_k$ for all $k \in [N]_{-i}$.

Consider the case of $\rho = 1$. Fix some $k \in [N]_{-i}$. We have
\begin{align*}
\Pr(W_{k,t} = 1) = \E \left[ \frac{1}{N} \left( a_1 + a_2 \loss_{k,t} - \frac{a_2}{N-1} \sum_{j \in [N]_{-k}} \loss_{j,t} \right) \right] .
\end{align*}
Since $a_2 > 0$, expert $i$'s optimal strategy for maximizing $\Pr(W_{k,t} = 1)$ is to minimize its subjective expectation of $\loss_{i,t}$. It follows that expert $i$ reports $r_{i,t} = b_{i,t}$.

Next, consider the case of $\rho < 1$. Then we have
\begin{align*}
\Pr(W_{0,t} = 1) = \E \left[ 1 - a_1 - \frac{a_2 (1 - \rho)}{N} \sum_{j=1}^N \loss_{j,t} \right] ,
\end{align*}
Since $1 - \rho > 0$ and $a_2 > 0$, we again have that expert $i$ reports $r_{i,t} = b_{i,t}$.
\end{proof}

\begin{proposition} \label{prop:woe-ci-c0}
Let $c_i$ and $c_0$ be defined as in \eqref{eqn:woe-ci-ck-c0}.
It holds that $c_i < c_0$.
\end{proposition}

\begin{proof}
Later in the proof, we will use the claim that for any $k \in [N]$ and any $s \in [T']$, the random variable $W_{k,s}$ is equal to zero with positive $\Pr_t$-probability. 
Indeed, inspecting \eqref{eqn:woe-generalized-elf}, since $a_1 \leq 1$ and $a_2 \leq 1$, it follows for any $k \in [N]$ that $\Pr(W_t = k) \leq \frac{2}{N}$.

Adopting the notation
\begin{align*}
\Lambda 
:= 
\min_{j \in [N]_{-i}} \left\{
    W_{j,0} - W_{i,0}
    + \sum_{s \in [T']_{-t}} W_{j,s} 
    - \sum_{s \in [T']_{-t}} W_{i,s}
\right\} ,
\end{align*}
We may rewrite $c_i$ and $c_0$ respectively as
\begin{align*}
c_i = \Pr_t \left( \Lambda > 1 \right)
\qquad
c_0 = \Pr_t \left( \Lambda > 0 \right) .
\end{align*}

With this notation, proving that $c_i < c_0$ is equivalent to proving that
\begin{align*}
\Pr_t \left( 0 < \Lambda \leq 1 \right) > 0 .
\end{align*}

Now, observe that with positive $\Pr_t$-probability, all of the following happen simultaneously:
\begin{itemize}
\item $W_{k,s} = 0$ for all $s \in [T']_{-t}$ and $k \in [N]$;
\item $W_{j,0} \in \left[ \frac{1}{8 N}, \frac{1}{4 N} \right]$ for all $j \in [N]_{-i}$;
\item $W_{i,0} \in \left[ -\frac{1}{4 N}, 0 \right]$.
\end{itemize}

Therefore, with positive $\Pr_t$-probability, for all $j \in [N]_{-i}$ simultaneously, we have that
\begin{align*}
W_{j,0} - W_{i,0}
+ \sum_{s \in [T']_{-t}} W_{j,s} 
- \sum_{s \in [T']_{-t}} W_{i,s}
\end{align*}
falls in the range $\left[ \frac{1}{8 N}, \frac{3}{8 N} \right]$. Hence, $\Pr_t \left( 0 < \Lambda \leq 1 \right) > 0$, as desired.
\end{proof}

\subsection{Bandit setting} \label{app:ic-bandit}

In the full-information setting, we proved incentive compatibility for a more general class of mechanisms that contains \elfRH{}. In the bandit setting, it is unclear if such a generalization makes sense, so we simply prove incentive compatibility for \banditelfRH{} itself.

\begin{proof}[of Theorem~\ref{thm:bandit-IC}]
Denote by $\cE_t$ the exploration rounds among rounds $1$ through $t$. Note that for any exploration round $s$, we have $C_s = I_s$. 
For any round $s$ that is not an exploration round, define $W_{j,s} = 0$ for all $j \in [N]$. This way, we may always sum the $W_{j,s}$ variables over all rounds $s$ rather than just the exploration rounds.

It suffices to show, for an arbitrary ``decision round'' $t \in [T]$ and an arbitrary ``target round'' $T' \in \{t, t + 1, \ldots, T\}$, that any expert $i \in [N]$, having observed $\cE_{t-1}$ and $(I_s, r_{I_s,s}, o_s)_{s \in [t-1]}$, strictly maximizes its subjective probability
\begin{align*}
\Pr \left( \sum_{s=0}^{T'} W_{i,s} 
           < \min_{j \in [N] \setminus \{i\}} \sum_{s=0}^{T'} W_{j,s} 
           \bigmid \cE_{t-1}, (I_s, r_{I_s,s}, o_s)_{s \in [t-1]} 
    \right)
\end{align*}
by truthfully selecting $r_{i,t} = b_{i,t}$. Indeed, if we can show this for arbitrary $T'$ as above, then the expert reports truthfully regardless of the target round, so Bandit Online IC-BI follows. 

Fix such a pair $(t, T')$ as well as a history 
$\cF_{t-1} := (\cE_{t-1}, (I_s, r_{I_s,s}, o_s)_{s \in [t-1]})$. 
The rest of the proof always conditions on this history. To avoid clutter, let $\Pr_t(\cdot)$ denote the conditional expectation $\Pr(\cdot \mid \cF_{t-1})$.

Adopting the notation $[N]_{-i} := [N] \setminus \{i\}$ and letting $\sum_{s \neq t}$ be shorthand for $\sum_{s \in \{0, 1, \ldots, T'\} \setminus \{t\}}$, we rewrite the above probability as
\begin{align}
&\Pr_t \left( \forall j \in [N]_{-i} \colon \sum_{s \neq t} W_{i,s} + W_{i,t} < \sum_{s \neq t} W_{j,s} + W_{j,t} \right) \nonumber \\
&= \Pr_t \left( \forall j \in [N]_{-i} \colon W_{i,t} < \sum_{s \neq t} W_{j,s} + W_{j,t} - \sum_{s \neq t} W_{i,s} \right) \nonumber \\
&= \Pr_t \left( W_{i,t} = 1 , \left[ \forall j \in [N]_{-i} \colon 1 < \sum_{s \neq t} W_{j,s} - \sum_{s \neq t} W_{i,s} \right] \right) \nonumber \\
&\quad+ \sum_{k \in [N]_{-i}} \Pr_t \left( W_{k,t} = 1, \left[ \forall j \in [N]_{-i} \colon 0 < \sum_{s \neq t} W_{j,s} + \ind{j = k} - \sum_{s \neq t} W_{i,s} \right] \right) \nonumber \\
&\quad+ \Pr_t \left( W_{0,t} = 1, \left[ \forall j \in [N]_{-i} \colon 0 < \sum_{s \neq t} W_{j,s} - \sum_{s \neq t} W_{i,s} \right] \right) \nonumber \\
&= \Pr_t \left( W_{i,t} = 1 \right) \cdot \Pr_t \left( \forall j \in [N]_{-i} \colon 1 < \sum_{s \neq t} W_{j,s} - \sum_{s \neq t} W_{i,s} \bigmid W_{i,t} = 1 \right) \label{eqn:bandit-woe-pr-term-type-i} \\
&\quad+ \sum_{k \in [N]_{-i}} \Pr_t \left( W_{k,t} = 1 \right) \cdot \Pr_t \left( \forall j \in [N]_{-i} \colon 0 < \sum_{s \neq t} W_{j,s} + \ind{j = k} - \sum_{s \neq t} W_{i,s} \bigmid W_{k,t} = 1 \right) \label{eqn:bandit-woe-pr-term-type-k} \\
&\quad+ \Pr_t \left( W_{0,t} = 1 \right) \cdot \Pr_t \left( \forall j \in [N]_{-i} \colon 0 < \sum_{s \neq t} W_{j,s} - \sum_{s \neq t} W_{i,s} \bigmid W_{0,t} = 1 \right) . \label{eqn:bandit-woe-pr-term-type-0}
\end{align}

Recall that $\Pr_t$ means we condition on $\cF_{t-1}$, which contains the set of prior exploration rounds and the bandit history up until the end of round $t-1$. 
Now, from belief independence, as we now explain, we can drop this conditioning on all terms of the type $\Pr_t(W_{j,t} = 1)$ in equations \eqref{eqn:bandit-woe-pr-term-type-i}, \eqref{eqn:bandit-woe-pr-term-type-k}, and \eqref{eqn:bandit-woe-pr-term-type-0}. Indeed, we have that $W_{j,t}$ is distributed according to $\mathrm{Bernoulli}\left( \frac{\varepsilon}{N} \left( \frac{1}{2} + \frac{1}{4} \loss_{j,t}) \right) \right)$ since \emph{(a)} whether or not round $t$ is an exploration round is determined by independent randomness and \emph{(b)} whether or not $C_t = j$ is determined by independent randomness (conditional on round $t$ being an exploration round). The loss $\loss_{j,t}$ only depends on $r_{j,t}$ and $o_t$, and it follows from belief independence that $r_{j,t}$ and $o_t$ are independent\footnote{In the case of $j = i$, we need only consider $o_t$ as expert $i$'s beliefs do not cover its report $r_{i,t}$.} (according to expert $i$'s subjective belief distribution) of the past reports and outcomes, independent of which previous rounds were exploration rounds, and consequently also independent of which experts were selected in previous rounds. 
Next, belief independence also allows us to drop the conditioning on $W_{j,t} = 1$ in the second probability term in each of equations \eqref{eqn:bandit-woe-pr-term-type-i}, \eqref{eqn:bandit-woe-pr-term-type-k}, and \eqref{eqn:bandit-woe-pr-term-type-0}. 
To see why, observe that if $Y_1$, $Y_2$, and $Y_3$ are random variables such that each of $Y_1$ and $Y_2$ is independent of $Y_3$. Then we have $\Pr(Y_1 \mid Y_2, Y_3) = \Pr(Y_1 \mid Y_2)$. As an example, we can apply this in \eqref{eqn:bandit-woe-pr-term-type-i} by taking 
$Y_1 = \ind{\forall j \in [N]_{-i} \colon 1 < \sum_{s \neq t} W_{j,s} - \sum_{s \neq t} W_{i,s}}$, 
$Y_2 = (\cE_{t-1}, (I_s, r_{I_s,s}, o_s)_{s \in [t-1})$, 
and $Y_3 = \ind{W_{i,t} = 1}$; note that $Y_1$ (similarly, $Y_2$) is independent of $Y_3$ from belief independence.

These simplifications yield
\begin{align}
\begin{split}
&\Pr \left( \sum_{s=0}^{T'} W_{i,s} < \min_{j \neq i} \sum_{s=0}^{T'} W_{j,s} \bigmid \cE_{t-1}, (I_s, r_{I_s,s}, o_s)_{s \in [t-1]} \right) \\
&= \Pr \left( W_{i,t} = 1 \right) \cdot \Pr_t \left( \forall j \in [N]_{-i} \colon 1 < \sum_{s \neq t} W_{j,s} - \sum_{s \neq t} W_{i,s} \right) \\
&\quad + \sum_{k \in [N]_{-i}} \Pr \left( W_{k,t} = 1 \right) \cdot \Pr_t \left( \forall j \in [N]_{-i} \colon 0 < \sum_{s \neq t} W_{j,s} + \ind{j = k} - \sum_{s \neq t} W_{i,s} \right) \\
&\quad + \Pr \left( W_{0,t} = 1 \right) \cdot \Pr_t \left( \forall j \in [N]_{-i} \colon 0 < \sum_{s \neq t} W_{j,s} - \sum_{s \neq t} W_{i,s} \right) \\
&= \Pr \left( W_{i,t} = 1 \right) \cdot \underbrace{\Pr_t \left( \forall j \in [N]_{-i} \colon \sum_{s \neq t} W_{j,s} - \sum_{s \neq t} W_{i,s} > 1 \right)}_{c_i} \\
&\quad + \sum_{k \in [N]_{-i}} \Pr \left( W_{k,t} = 1 \right) \cdot \underbrace{\Pr_t \left( \forall j \in [N]_{-i} \colon \sum_{s \neq t} W_{j,s} - \sum_{s \neq t} W_{i,s} > -\ind{j = k} \right)}_{c_k} \\
&\quad + \Pr \left( W_{0,t} = 1 \right) \cdot \underbrace{\Pr_t \left( \forall j \in [N]_{-i} \colon \sum_{s \neq t} W_{j,s} - \sum_{s \neq t} W_{i,s} > 0 \right)}_{c_0} ,
\end{split} \label{eqn:bandit-woe-ci-ck-c0}
\end{align}
where the last line is just a natural rearrangement and the notation $c_i$ and $c_k$ is for convenience. 

We rewrite the above as
\begin{align}
&c_i \cdot \left( 1 - \sum_{k \in [N]_{-i}} \Pr(W_{k,t} = 1)  - \Pr(W_{0,t} = 1) \right) + \sum_{k \in [N]_{-i}} c_k \Pr(W_{k,t} = 1) + c_0 \Pr(W_{0,t} = 1) \nonumber \\
&= c_i - \sum_{k \in [N]_{-i}} (c_i - c_k) \Pr(W_{k,t} = 1) - (c_i - c_0) \Pr(W_{0,t} = 1) .
\label{eqn:bandit-expert-i-incentive-objective}
\end{align}

Recall that the expert's goal is to select a report $r_{i,t}$ that maximizes the above expression. For all $k \in [N] \setminus \{i\}$, it clearly holds that $c_i \leq c_0 \leq c_k$. In Proposition~\ref{prop:bandit-woe-ci-c0} (stated and proved immediately after this proof), we can and will show something even stronger: it holds that $c_i < c_0$ and hence also $c_i < c_k$ for all $k \in [N] \setminus \{i\}$.

Now, since
\begin{align*}
\Pr(W_{0,t} = 1) = \E \left[ 1 - \frac{\varepsilon}{2} - \frac{\varepsilon}{4 N} \sum_{j=1}^N \loss_{j,t} \right] ,
\end{align*}
it follows that expert $i$ reports $r_{i,t} = b_{i,t}$.
\end{proof}

\begin{proposition} \label{prop:bandit-woe-ci-c0}
Let $c_i$ and $c_0$ be defined as in \eqref{eqn:bandit-woe-ci-ck-c0}.
It holds that $c_i < c_0$.
\end{proposition}

\begin{proof}
The proof is essentially identical the proof of Proposition~\ref{prop:woe-ci-c0}. To avoid excessive redundancy, we just mention the differences:
\begin{itemize}
\item The first paragraph of the proof of Proposition~\ref{prop:woe-ci-c0} is now simpler as we need not consider a general family of mechanisms. We still make the same claim (which is now immediate) that for any $k \in [N]$ and any $s \in [T']$, 
the random variable $W_{k,s}$ is equal to zero with positive $\Pr_t$-probability.
\item We note that the conditioning in the $\Pr_t$ from the proof of Proposition~\ref{prop:woe-ci-c0} is different from the conditioning in the $\Pr_t$ of the present proof.
\item From the second paragraph onwards, the proofs are the same, except the noise is at scale $\frac{\varepsilon}{4 N}$ instead of $\frac{1}{4 N}$.
\end{itemize}
\end{proof}

\subsection{\elfRH{} with multiple draws is not \truthful{}} \label{app:multiple-draws-not-ic}

Before going into the formal proof of Theorem~\ref{thm:multiple-draws-not-ic}, we give intuition for why incentive compatibility breaks when multiple sad points can be awarded in a round. 
It suffices to consider the case of $N=2$ and $T=2$ with two experts we refer to as $i$ and $-i$. 
The mathematical reason is that the objective that expert $i$ wants to minimize involves a product of the loss of expert $i$ and the other expert $-i$:
\begin{align}\label{eqn:product-of-expectations-deriv}
\E \bigl[ \loss(r_{i,1}, o_1) \cdot \loss(r_{-i,1}, o_1) \bigr],
\end{align}
which we can find in Eq.~\eqref{eqn:product-of-expectations-new} of the formal proof. Such a term never appears in the case of the single draw case because the draw of $C_t$ splits the world into two worlds --- where the loss $\loss_{i,t}$ matters in one world and the loss $\loss_{-i,t}$ matters in the other world --- and the product of them does not matter. 
Assume that expert $i$'s report $r_{i,1}$ is meant to maximize the probability that the expert selected in round $3$. We can create a belief that $o_1 \sim \mathrm{Bernoulli}(1/2)$, $r_{-i,1}=1$ with probability 1, but a negative product term Eq.~\eqref{eqn:product-of-expectations-deriv} exists. 
By definition, the truthful report is $r_{i,1} = 1/2$.
However, the product term indicates that having both $\loss_{i,1}$ and $\loss_{i,2}$ high (i.e., both experts perform poorly) works against expert $i$. Expert $i$ is incentivized to report $r_{i,1} < 1/2$ because by doing so, either $\loss_{i,1}$ or $\loss_{i,2}$ is small regardless of whether $o_1 = 0$ or $o_1 = 1$. 
The negative product term of Eq.~\eqref{eqn:product-of-expectations-deriv} is derived from the asymmetry between leading and being led. To make the story simple, let us discuss the count of sad points. There are three cases at the end of round $t=1$. Namely,
\begin{itemize}
\item Expert $i$ takes the lead: $\woe_{i,1} < \woe_{-i,1}$.
\item Tie: $\woe_{i,1} = \woe_{-i,1}$.
\item Expert $-i$ takes the lead: $\woe_{i,1} > \woe_{-i,1}$.
\end{itemize}
Note that the sad points are recalculated at each round, but it is correlated with the loss $\loss_{i,t}$.
If expert $i$ believes that expert $-i$ will perform extremely well in round $2$, expert $i$'s valuation of a tie is not very high because, with a tie, expert $i$ cannot take the lead at round $t=2$ due to the good performance of expert $-i$ at round $2$. Conversely, if expert $i$ takes a lead at $t=1$, the odds of winning for expert $i$ are significantly increased. Even when only expert $i$ receives a sad point at round $2$, it is fifty-fifty. 
In summary, against a strong competitor, expert $i$ takes a risk to increase anti-correlation with the report of the other expert $-i$, which results in a report that 
does not solely maximize the expected loss.

\begin{proof}[of Theorem~\ref{thm:multiple-draws-not-ic}]

Since the game only has two rounds, it can be difficult to remember that $w_{1,2}$ corresponds to expert 1 in round 2 while $w_{2,1}$ corresponds to expert 2 in round 1. So, for clarity, we will consider the report of an expert $i$ and let $-i$ refer to the other expert. For example, we will write things like $w_{i,2}$ and $w_{-i,1}$.

The proof is somewhat long, so we first provide an overview. 
We show that for any strictly proper loss with range $[0, 1]$, there exists a belief distribution for expert $i$ such that honest reporting in round 1 does not maximize the probability that expert $i$ is selected in round 3. Concretely, this belief distribution has the following structure:
\begin{itemize}
\item $o_1 \sim \mathrm{Bernoulli}(1/2)$;
\item either $r_{-i,1} = 1$ with probability 1 or $r_{-i,1} = 0$ with probability 1;\footnote{Regarding whether $r_{-i,1} = 1$ with probability 1 or $r_{-i,1} = 0$ with probability 1, the first choice works if the condition $\loss(1,1) < \loss(1, 0)$ holds while the second choice works if the condition $\loss(0, 0) < \loss(0, 1)$ holds; we show that it must be the case that at least one of these conditions holds.}
\item $o_2 \sim \mathrm{Bernoulli}(1/2)$;
\item conditional on $o_2$, it holds with probability 1 that $r_{-i,2} = o_2$.
\end{itemize}

Next, explicitly working out the probability that expert $i$ is selected in round 3 results in a sum of terms, with each term being an expected value. In isolation, all terms but one are minimized via honest reporting. The remaining term involves an expectation of $\loss(r_{i,1},o_1) \cdot \loss(r_{-i,1},o_1)$. Both losses depend on the outcome $o_1$ in such a way that, for the particular belief distribution specified above, this term is not minimized from honest reporting.

We now present the formal proof.

Define
\begin{align*}
\mathcal{W}_{\mathrm{win}} := \left\{ w \in \{0, 1\}^2 \colon \sum_{t=1}^2 w_{i,t} < \sum_{t=1}^2 w_{-i,t} \right\}
\end{align*}
and
\begin{align*}
\mathcal{W}_{\mathrm{tie}} := \left\{ w \in \{0, 1\}^2 \colon \sum_{t=1}^2 w_{i,t} = \sum_{t=1}^2 w_{-i,t} \right\}
\end{align*}

Note that ties are broken uniformly at random; this can be accomplished by using small initial uniform random noise $W_{i,0}$ and $W_{-i,0}$, as we did in \elfRH{}. 
The probability of expert $i$ being selected in round 3 is:
\begin{align*}
&\sum_{w \in \mathcal{W}_{\mathrm{win}}} \Pr( W_i = w_i , W_{-i} = w_{-i}) 
+ \frac{1}{2} \sum_{w \in \mathcal{W}_{\mathrm{tie}}} \Pr( W_i = w_i , W_{-i} = w_{-i}) \\
&= \Pr(W_i = (0, 0), W_{-i} = (1, 1)) \\
&+ \Pr(W_i = (0, 0), W_{-i} = (0, 1)) \\
&+ \Pr(W_i = (0, 0), W_{-i} = (1, 0)) \\
&+ \Pr(W_i = (0, 1), W_{-i} = (1, 1)) \\
&+ \Pr(W_i = (1, 0), W_{-i} = (1, 1)) \\
&+ \frac{1}{2} \Pr(W_i = (0, 0), W_{-i} = (0, 0)) \\
&+ \frac{1}{2} \Pr(W_i = (0, 1), W_{-i} = (0, 1)) \\
&+ \frac{1}{2} \Pr(W_i = (0, 1), W_{-i} = (1, 0)) \\
&+ \frac{1}{2} \Pr(W_i = (1, 0), W_{-i} = (0, 1)) \\
&+ \frac{1}{2} \Pr(W_i = (1, 0), W_{-i} = (1, 0)) \\
&+ \frac{1}{2} \Pr(W_i = (1, 1), W_{-i} = (1, 1)) .
\end{align*}

Since we draw lotteries for each expert, we have
\begin{align*}
\Pr(W_{i,t} = 1) = \frac{1}{2} + \frac{1}{4} \loss_{i,t} \quad \text{for all } i \in [N] .
\end{align*}
Let $a_1 = \frac{1}{2} \loss_{i,1}$ and $a_2 = \frac{1}{2} \loss_{i,2}$. Similarly, let $b_1 = \frac{1}{2} \loss_{-i,1}$ and $b_2 = \frac{1}{2} \loss_{-i,2}$. 
Hence, we have 
\begin{align*}
\Pr(W_i = (0, 0), W_{-i} = (1, 1)) 
= \frac{1}{2} \left( 1 - a_1 \right) 
  \cdot \frac{1}{2} \left( 1 - a_2 \right) 
  \cdot \frac{1}{2}\left( 1 + b_1 \right) 
  \cdot \frac{1}{2}\left( 1 + b_2 \right) ,
\end{align*}
and the other terms are calculated in the same way. 
Therefore, the probability of expert $i$ being selected in round 3 is equal to $\frac{1}{16}$ times the expectation of
\begin{align*}
&(1 - a_1) (1 - a_2) (1 + b_1) (1 + b_2) \\
&+ (1 - a_1) (1 - a_2) (1 - b_1) (1 + b_2) \\
&+ (1 - a_1) (1 - a_2) (1 + b_1) (1 - b_2) \\
&+ (1 - a_1) (1 + a_2) (1 + b_1) (1 + b_2) \\
&+ (1 + a_1) (1 - a_2) (1 + b_1) (1 + b_2) \\
&+ \frac{1}{2} (1 - a_1) (1 - a_2) (1 - b_1) (1 - b_2) \\
&+ \frac{1}{2} (1 - a_1) (1 + a_2) (1 - b_1) (1 + b_2) \\
&+ \frac{1}{2} (1 - a_1) (1 + a_2) (1 + b_1) (1 - b_2) \\
&+ \frac{1}{2} (1 + a_1) (1 - a_2) (1 - b_1) (1 + b_2) \\
&+ \frac{1}{2} (1 + a_1) (1 - a_2) (1 + b_1) (1 - b_2) \\
&+ \frac{1}{2} (1 + a_1) (1 + a_2) (1 + b_1) (1 + b_2) .
\end{align*}

The above simplifies to
\begin{align*}
a_1 b_1 (-a_2 + b_2) 
+ a_2 b_2 (-a_1 + b_1) 
- 3 (a_1 + a_2) 
+ 3 (b_1 + b_2) 
+ 8 ,
\end{align*}
which, by definition of $a_1$, $b_1$, etc.~is equal to one eighth times
\begin{align*}
\loss_{i,1} \cdot \loss_{-i,1} \cdot (-\loss_{i,2} + \loss_{-i,2}) 
+ \loss_{i,2} \cdot \loss_{-i,2} \cdot (-\loss_{i,1} + \loss_{-i,1}) 
- 6 (\loss_{i,1} + \loss_{i,2}) 
+ 6 (\loss_{-i,1} + \loss_{-i,2}) 
+ 64 .
\end{align*}

Since our focus is on how expert $i$ selects its report $r_{i,1}$, it suffices to restrict focus to terms involving $\loss_{i,1}$, which are
\begin{align*}
\loss_{i,1} \cdot \loss_{-i,1} \cdot (-\loss_{i,2} + \loss_{-i,2}) 
- \loss_{i,1} \cdot \loss_{i,2} \cdot \loss_{-i,2} 
- 6 \loss_{i,1} 
\end{align*}

Expert $i$ selects its report $r_{i,1}$ to maximize its (subjective) expectation of the above quantity. From belief independence, this expectation can be written as 
\begin{align}
\E \left[ \loss_{i,1} \cdot \loss_{-i,1} \right] \cdot \E \left[ -\loss_{i,2} + \loss_{-i,2} \right] 
- \E \left[ \loss_{i,1} \right] \cdot \E \left[ \loss_{i,2} \cdot \loss_{-i,2} \right] 
- 6 \E \left[ \loss_{i,1} \right] . \label{eqn:overall-subjective-expectation-new}
\end{align}
From the strict properness of the loss, the third term is maximized by taking $r_{i,1} = b_{i,1}$. Since the loss is nonnegative, $\E \left[ \loss_{i,2} \cdot \loss_{-i,2} \right]$ is nonnegative, and so the second term also is maximized by taking $r_{i,1} = b_{i,1}$. The first term, however, is troublesome, due to the multiplicative interaction between $\loss_{i,1}$ and $\loss_{-i,1}$. To see why, for clarity, let us write the losses in the first term explicitly: 
\begin{align}
\E \bigl[ \loss(r_{i,1}, o_1) \cdot \loss(r_{-i,1}, o_1) \bigr] \cdot \E \bigl[ -\loss(r_{i,2}, o_2) + \loss(r_{-i,2}, o_2) \bigr] . \label{eqn:product-of-expectations-new}
\end{align}

We will first show that expert $i$ has a belief distribution such that $\E \bigl[ -\loss(r_{i,2}, o_2) + \loss(r_{-i,2}, o_2) \bigr] < 0$, after which we will demonstrate that the expression in \eqref{eqn:product-of-expectations-new} is maximized by some $r_{i,1}$ not equal to $b_{i,1}$. It should be clear that if \eqref{eqn:product-of-expectations-new} is not maximized by taking $r_{i,1} = b_{i,1}$, then the overall subjective expectation \eqref{eqn:overall-subjective-expectation-new} also is not maximized by taking $r_{i,1} = b_{i,1}$.

Now, let us suppose that expert $i$ has the following subjective belief:
\begin{itemize}
\item $o_1 \sim \mathrm{Bernoulli}(1/2)$;
\item $r_{-i,1} = 1$ with probability 1;
\item $o_2 \sim \mathrm{Bernoulli}(1/2)$;
\item conditional on $o_2$, it holds with probability 1 that $r_{-i,2} = o_2$.
\end{itemize}
Then for any report $r_{i,2} \in [0, 1]$, 
\begin{align*}
&\E \bigl[ -\loss(r_{i,2}, o_2) + \loss(r_{-i,2}, o_2) \bigr] \\
&= \frac{1}{2} \bigl( -\loss(r_{i,2}, 1) + \loss(1, 1) \bigr)
   + \frac{1}{2} \bigl( -\loss(r_{i,2}, 0) + \loss(0, 0) \bigr) \\
&= \frac{1}{2} \bigl( \loss(1, 1) + \loss(0, 0) \bigr)
   - \frac{1}{2} \bigl( \loss(r_{i,2}, 1) + \loss(r_{i,2}, 0) \bigr) .
\end{align*}
Now, from the strict monotonicity of the partial losses $r \mapsto \loss(r, 1)$ and $r \mapsto \loss(r, 0)$ (see Proposition 6 of \citet{williamson2016composite}), it follows that the above expression must be negative. 

It remains to show that under the same belief, the report $r_{i,1}$ that minimizes
\begin{align*}
\E \bigl[ \loss(r_{i,1}, o_1) \cdot \loss(r_{-i,1}, o_1) \bigr]
\end{align*}
is not equal to $\frac{1}{2}$ (which is $b_{i,1}$). Taking the expectation, the above expression is equal to
\begin{align*}
\frac{1}{2} \left( \loss(r_{i,1}, 1) \cdot \loss(1, 1) + \loss(r_{i,1}, 0) \cdot \loss(1, 0) \right) .
\end{align*}

Next, from Proposition~\ref{prop:proper-inequalities} (stated and proved immediately after this proof), we will use the fact that for any strictly proper loss, at least one of $\loss(1, 1) < \loss(1, 0)$ or $\loss(0, 0) < \loss(0, 1)$ must be true. 
Without loss of generality\footnote{A symmetric argument can be used if instead only $\loss(0, 0) < \loss(0, 1)$ holds; the only modification necessary is change expert $i$'s belief so that $r_{-i,1} = 0$ with probability 1.}, assume that $\loss(1, 1) < \loss(1, 0)$. Then
\begin{align*}
\frac{1}{2} \left( \loss(r_{i,1}, 1) \cdot \loss(1, 1) + \loss(r_{i,1}, 0) \cdot \loss(1, 0) \right)
\end{align*}
is proportional to
\begin{align*}
\loss(r_{i,1}, 1) \cdot \frac{\loss(1, 1)}{\loss(1, 0)} + \loss(r_{i,1}, 0) ,
\end{align*}
which, defining $C := \frac{\loss(1, 1)}{\loss(1, 0)} < 1$, is proportional to
\begin{align*}
\frac{C}{1 + C} \loss(r_{i,1}, 1) + \frac{1}{1 + C} \loss(r_{i,1}, 0) .
\end{align*}
From the strict properness of the loss, it follows that the above expression is minimized by taking $r_{i,1} = \frac{C}{1 + C}$. Since $C < 1$, we have $r_{i,1} < \frac{1}{2}$. 
This concludes our proof that the mechanism with multiple draws is not Online IC-BI.
\end{proof}

\begin{proposition} \label{prop:proper-inequalities}
If $\loss$ is strictly proper, then at least one of $\loss(1, 1) < \loss(1, 0)$ or $\loss(0, 0) < \loss(0, 1)$ must be true.
\end{proposition}

\begin{proof}[of Proposition~\ref{prop:proper-inequalities}]
Suppose for a contradiction that $\loss(1, 1) \geq \loss(1, 0)$ and $\loss(0, 0) \geq \loss(0, 1)$. We consider three exhaustive cases, each of which gives a contradiction to the strictly proper loss's strict monotonicity property:
\begin{itemize}
\item In case 1, $\loss(0, 0) = \loss(1, 1)$. We then have $\loss(0, 0) = \loss(1, 1) \geq \loss(1, 0)$, contradicting the strict monotonicity of $r \mapsto \loss(r, 0)$ (see Proposition 6 of \citet{williamson2016composite}).
\item In case 2, $\loss(0, 0) > \loss(1, 1)$. This implies that $\loss(0, 0) > \loss(1, 1) \geq \loss(1, 0)$, again contradicting the strict monotonicity of $r \mapsto \loss(r, 0)$.
\item In case 3, $\loss(0, 0) < \loss(1, 1)$. We then have $\loss(1, 1) > \loss(0, 0) \geq \loss(0, 1)$, contradicting the strict monotonicity of $r \mapsto \loss(r, 1)$.
\end{itemize}
Therefore, either $\loss(1, 1) < \loss(1, 0)$ holds or $\loss(0, 0) < \loss(0, 1)$ holds (or both hold).
\end{proof}

\section{Regret analysis proofs} \label{app:regret}

\begin{proof}[of Proposition~\ref{prop:cond-rh-to-fh}]
Fix a round $t$ and a previous round $s \in [t-1]$. We claim that the law of $\Woe_s$ is the same whether running \banditelfRH{} or \banditelfFH{}. Indeed, under either algorithm, $\Woe_s$ can be formed via the following generative process: \emph{(i)} with probability $1 - \varepsilon$, do not let any expert participate in a lottery (set $\Woe_s = \mathbf{0}$); \emph{(ii)} with remaining probability $\varepsilon$, select an expert $C_s$ uniformly at random from $[N]$, set $\Woe_{C_s,s} = 1$ with probability $\frac{1}{2} + \frac{1}{4} \loss_{C_s,s}$, and set $\Woe_{j,s} = 0$ for all $j \neq C_s$. Now, since each $\Woe_s$ has the same law under either algorithm, it follows that for all $j \in [N]$, the sum $\sum_{s=0}^{t-1} \Woe_{j,s}$ has the same law under either algorithm.
\end{proof}

\begin{proof}[of Lemma~\ref{lemma:E-max-abs}]
For each $i$ and $s$, let $X'_{i,s}$ be an independent copy of $X_{i,s}$. For any random variable $A$, let $\E_A$ denote the expectation with respect to $A$ (conditional on everything else). From a standard symmetrization argument with independent Rademacher random variables\footnote{A Rademacher random variable takes values $+1$ and $-1$ with equal probability.} $\sigma_1, \ldots, \sigma_t$, we have
\begin{align*}
\E \left[ \max_{j \in [N]} \left| \sum_{s=1}^t X_{j,s} \right| \right] 
&= \E_X \left[ \max_{j \in [N]} \left| \sum_{s=1}^t \left( X_{j,s} - \E_{X'} \left[ X'_{j,s} \right] \right) \right| \right] \\
&\leq \E_X \E_{X'} \left[ \max_{j \in [N]} \left| \sum_{s=1}^t ( X_{j,s} - X'_{j,s} ) \right| \right] & \text{(Jensen's inequality)} \\
&= \E_X \E_{X'} \E_{\bm{\sigma}} \left[ \max_{j \in [N]} \left| \sum_{s=1}^t \sigma_s ( X_{j,s} - X'_{j,s} ) \right| \right] & \textcolor{blue}{(\star)} \\
&\leq 2 \E_X \E_{\bm{\sigma}}  \left[ \max_{j \in [N]} \left| \sum_{s=1}^t \sigma_s X_{j,s} \right| \right] \\
&= 2 \E \left[ \E_{\bm{\sigma}} \left[ \max_{j \in [N]} \left| \sum_{s=1}^t \sigma_s X_{j,s} \right| \right] \right] ,
\end{align*}
where the conditional expectation $\E_{\bm{\sigma}}$ is conditional on $X$. 
The step \textcolor{blue}{$(\star)$} is true because each $X_{j,s}$ and $X'_{j,s}$ are identically distributed, and hence
\begin{align*}
+1 \cdot (X_{j,s} - X'_{j,s}) \qquad \text{and} \qquad -1 \cdot (X_{j,s} - X'_{j,s})
\end{align*}
have the same distribution, which allows us to take an outer expectation over $\bm{\sigma} = (\sigma_1, \ldots, \sigma_s)$. 
Rewriting the last line above as
\begin{align*}
2 \E \left[ \E_{\bm{\sigma}} \left[ \max_{j \in [N]} \max \left\{ \sum_{s=1}^t \sigma_s \cdot X_{j,s}, \sum_{s=1}^t \sigma_s \cdot (-X_{j,s}) \right\} \right] \right] ,
\end{align*}
we can upper bound the conditional expectation using Massart's finite class lemma \cite[Lemma 5.2]{massart2000some} to get the upper bound
\begin{align}
\E \left[ \max_{j \in [N]} \left| \sum_{s=1}^t X_{j,s} \right| \right] 
\leq \sqrt{2 \log (2 N)} \cdot \E \left[ \max_{j \in [N]} \|X_j\|_2 \right] .\label{eqn:E-max-ell2-norm}
\end{align}

To see why this is progress, observe that with high probability, for all $j \in [N]$, we have that $X_j$ contains at most order $T q$ components whose sizes are of order $\frac{\babs}{q}$ (recall that $q$ can depend on both $N$ and $T$); the rest of the components' sizes are of constant order. To get a rigorous bound, we will get a high probability bound on $\max_{j \in [N]} \|X_j\|_2$. Recall that for each $j \in [N]$, the candidate indicators $C_{j,1}, \ldots, C_{j,t}$ are i.i.d.~Bernoulli random variables with success probability $\frac{1}{N}$. Therefore, from a multiplicative Chernoff bound\footnote{See, e.g., Theorem 2 of \citet{kuszmaul2021multiplicative}.}, for any $\alpha > 0$, it holds that
\begin{align*}
\Pr \left( \sum_{s=1}^t C_{j,s} \geq (1 + \alpha) t q \right) \leq \exp \left( -\frac{\alpha^2 t q}{2 + \alpha} \right) .
\end{align*}
Setting $\alpha = 1$ and using the union bound, we have
\begin{align*}
\Pr \left( \max_{j \in [N]} \sum_{s=1}^t C_{j,s} \geq 2 t q \right) \leq N \exp \left( -\frac{t q}{3} \right) .
\end{align*}
On the event $\max_{j \in [N]} \sum_{s=1}^t C_{j,s} < 2 t q$, we have $\max_{j \in [N]} \|X_j\|_2^2 \leq 2 t q \left( \frac{\babs}{q} \right)^2 + \left( t - 2 t q \right) \cdot \babs^2 \leq \frac{3 \babs^2 t}{q}$. 
Consequently,
\begin{align*}
\Pr \left( \max_{j \in [N]} \|X_j\|_2 \geq \babs \sqrt{\frac{3 t}{q}} \right)
 \leq N \exp \left( -\frac{t q}{3} \right) \leq \sqrt{q} ,
\end{align*}
where the second inequality uses the assumption that $t \geq \frac{3}{q} \log \left( \frac{N}{\sqrt{q}} \right)$. 
Using this result in the RHS of \eqref{eqn:E-max-ell2-norm} implies that
\begin{align*}
\E \left[ \max_{j \in [N]} \left| \sum_{s=1}^t X_{j,s} \right| \right] 
&\leq \sqrt{2 \log (2 N)} \cdot \left( \left( 1 - \sqrt{q} \right) \babs \sqrt{\frac{3 t}{q}} + \sqrt{q} \cdot \frac{\babs}{q} \sqrt{t} \right) \\
&\leq \babs \sqrt{2 \log (2 N)} \cdot \left( \sqrt{\frac{3 t}{q}} + \sqrt{\frac{t}{q}} \right) \\
&\leq 4 \babs \sqrt{\frac{t \log (2N)}{q}} ,
\end{align*}
which concludes the proof.
\end{proof}

\subsection{Proof of Theorem \ref{thm:lead-pack-prob-bound}}\label{subsec:lead-pack-prob-bound}

During the proof of Theorem \ref{thm:lead-pack-prob-bound} below, we assume the initial noise $\Woe_0$ is fixed to $\woe_0$. Every probability and expectation is conditional to this; $\Przero[\cdot], \Ezero[\cdot]$ indicate $\Pr[\cdot | \Woe_0 = \woe_0], \E[\cdot |  \Woe_0 = \woe_0]$ in the scope of the theorem.

\begin{proof}[of Theorem \ref{thm:lead-pack-prob-bound}] 
To unify the analysis, let $\varepsilon = 1$ for the full-information case.
Let $L_{i,t} = \sum_{s=1}^t \loss_{i,s} + \frac{4N}{\varepsilon} w_{i,0}$ and $Z_{i,t} = \frac{\varepsilon}{4N} \sum_{s=1}^t X_{i,s}$. Then, 
\[
\sum_{t'=0}^{t-1} W_{i,t'} = 
\frac{\varepsilon}{4N} L_{i,t-1} + Z_{i,t-1} + 2
\]
and $Z_{i,t-1}$ is a centered (i.e., zero-mean) Poisson binomial random variable whose distribution has parameters lying in $[\varepsilon/(2 N), 3\varepsilon/(4 N)] =: [\lPoi, \uPoi]$. 
Let $q_{i,t}(k) = \mathbb{P}[k - 1 < Z_{i,t} \le k]$. 
There is at most one possible value of $Z_{i,t}$ that satisfies $k - 1 < Z_{i,t} \le k$. 

In the following, we show two lemmas that are directly derived from the Poisson binomial tail bounds in Section \ref{sec_poisson}. Note that this section uses a centered Poisson binomial random variable with mean zero, whereas Section \ref{sec_poisson} uses non-centered standard Poisson binomial random variables following the standard of the literature \citep{pollard_notebook_2021}.
\begin{lemma}[Small $q_{i,t}(k)$ for large deviation]\label{lem_smallpk}
\[ 
q_{i,t}(k) \le \frac{1}{(Nt/\lPoi)^2}
\]
for 
$k: |k| \ge \sqrt{8 \lPoi  t \log (Nt/\lPoi)}$.
\end{lemma}
\begin{proof}
Proof of this lemma is straightforward from Lemma \ref{lem_tail_upper} and $2 \lPoi \ge \bar{\theta}_i$.
\end{proof} %proof of Lemma \ref{lem_smallpk} 

\begin{lemma}[Results on ratio]\label{lem_ratio_inswitch}
There exist constants $C_1,C_2 > 0$ such that, 
\[
\frac{q_{i,t}(k-1)}{q_{i,t}(k)} \ge 1 - Q_{t,k}
\]
for any $k \in [- \lPoi t/24 + 2, \lPoi t/24 - 2]$, where 
$Q_{t,k} = \min \left\{ \frac{C_1 \log t}{\sqrt{\lPoi t}} + \frac{C_2 |k|}{\lPoi t}, 1 \right\}$.
\end{lemma}

\begin{proof}
This is derived in Lemma~\ref{lem_poibin_tail_ratio}.
Note that this is a generalized version of a similar bound for symmetric binomial random variables stated at the beginning of page 7 in \citet{devroye2013prediction}. 
\end{proof}

By using Lemmas \ref{lem_smallpk} and \ref{lem_ratio_inswitch}, the following derives Theorem \ref{thm:lead-pack-prob-bound}.

Let $\Rin = \left\{-\left\lceil \sqrt{8\lPoi t \log (Nt/\lPoi)} \right\rceil,-\left\lceil \sqrt{8\lPoi t \log (Nt/\lPoi)} \right\rceil+1,\dots, \left\lceil \sqrt{8\lPoi t \log (Nt/\lPoi)} \right\rceil\right\}$ and $\Rout = \{-t,-t+1,\dots,t-1,t\} \setminus \Rin$.

Intuitively speaking, the following discussion bounds the event $|A_t| = 1$ in the following way. Letting $j$ be the (unique) element of $A_t$, we split the event into two cases: The first case is $Z_{i,t}$ is very large, which is bounded via Lemma \ref{lem_smallpk}. 
The more important second case is that $Z_{i,t}$ is moderate. In this case, we bound the event via Lemma \ref{lem_ratio_inswitch}.

\begin{align}
\lefteqn{
\Przero \left(  |A_t| = 1 \right) 
}\\
& := \sum_{j=1}^N \Przero \left( \min_{i \neq j} \sum_{s=0}^{t-1} W_{i,s} > \sum_{s=0}^{t-1} W_{j,s} + 1 \right)\\ 
&= \sum_{j =1}^N \Przero \left( 
  \min_{i \neq j} \left\{ \frac{\varepsilon}{4 N} L_{i,t-1} + Z_{i,t-1}  \right\} > \frac{\varepsilon}{4 N} L_{j,t-1}  + Z_{j,t-1} + 1
\right) \\
&\geq \sum_{k=-t}^t \sum_{j =1}^N q_{j,t-1}(k) \Przero \left( 
  \min_{i \neq j} \left\{ \frac{\varepsilon}{4 N} L_{i,t-1}  + Z_{i,t-1} \right\} > \frac{\varepsilon}{4 N} L_{j,t-1}  + k + 1
\right) \\
&\ge \sum_{k\in \Rin}  \sum_{j =1}^N q_{j,t-1}(k) \Przero \left( 
  \min_{i \neq j} \left\{ \frac{\varepsilon}{4 N} L_{i,t-1}  + Z_{i,t-1} \right\} > \frac{\varepsilon}{4 N} L_{j,t-1}  + k + 1
\right)\\
&\ge \sum_{k\in \Rin}  \sum_{j =1}^N q_{j,t-1}(k) \Przero \left( 
  \min_{i \neq j} \left\{ \frac{\varepsilon}{4 N} L_{i,t-1}  + Z_{i,t-1} \right\} > \frac{\varepsilon}{4 N} L_{j,t-1}  + k + 1
\right) 
\label{ineq_riterm}\\
&\ \ +\hspace{-0.3em}\sum_{k\in \Rout} \sum_{j =1}^N q_{j,t-1}(k+1) \Przero \left( 
  \min_{i \neq j} \left\{ \frac{\varepsilon}{4 N} L_{i,t-1}  + Z_{i,t-1} \right\} > \frac{\varepsilon}{4 N} L_{j,t-1}  + k + 1
\right)\hspace{-0.3em}\left(1 - Q_{t-1,k+1})\right) 
- \frac{1}{N t} ,
\end{align}
where the last inequality holds because 
\begin{align}
&\sum_{k\in \Rout} \sum_{j =1}^N q_{j,t-1}(k+1) \Przero \left( 
  \min_{i \neq j} \left\{ \frac{\varepsilon}{4 N} L_{i,t-1}  + Z_{i,t-1} \right\} > \frac{\varepsilon}{4 N} L_{j,t-1}  + k + 1
\right)\left(1 - Q_{t-1,k+1})\right)
 \nonumber\\
&\le 
\sum_{k\in \Rout} \sum_{j =1}^N q_{j,t-1}(k+1) 
\nonumber\\
&\leq \sum_{k\in \Rout} \sum_{j =1}^N \frac{1}{(Nt/\lPoi)^2}  \tag{Lemma~\ref{lem_smallpk}} \\
&\leq \frac{1}{Nt} . \nonumber 
\end{align}

Moreover, Eq.\eqref{ineq_riterm} is transformed as 
\begin{align}
\lefteqn{
\sum_{k\in \Rin}  \sum_{j =1}^N q_{j,t-1}(k) \Przero \left( 
  \min_{i \neq j} \left\{ \frac{\varepsilon}{4 N} L_{i,t-1}  + Z_{i,t-1} \right\} > \frac{\varepsilon}{4 N} L_{j,t-1}  + k + 1
\right) 
}\\
&= \sum_{k\in \Rin}  \sum_{j =1}^N \frac{q_{j,t-1}(k)}{q_{j,t-1}(k+1)} q_{j,t-1}(k+1) \Przero \left( 
  \min_{i \neq j} \left\{ \frac{\varepsilon}{4 N} L_{i,t-1}  + Z_{i,t-1} \right\} > \frac{\varepsilon}{4 N} L_{j,t-1}  + k + 1
\right) \\
&\ge \sum_{k\in \Rin} \sum_{j =1}^N q_{j,t-1}(k+1) \Przero \left( 
  \min_{i \neq j} \left\{ \frac{\varepsilon}{4 N} L_{i,t-1}  + Z_{i,t-1} \right\} > \frac{\varepsilon}{4 N} L_{j,t-1}  + k + 1
\right)\hspace{-0.3em}\left(1 - Q_{t-1,k+1})\right). \text{\ \ \ \ (Lemma \ref{lem_ratio_inswitch})}
\end{align}
Here, the application of Lemma \ref{lem_ratio_inswitch} for all $k \in \Rin$ is valid because 
\begin{align}
|k| \le \left\lceil \sqrt{8 \lPoi t \log (Nt/\lPoi)} \right\rceil 
&< \lPoi t/24 - 24 + 1 \text{\ \ \ \ (by Eq.~\eqref{ineq_tmin})}\\
&<  \lPoi t/24 - 2,
\end{align}
which satisfies the premise of Lemma \ref{lem_ratio_inswitch} when using $k+1$ and $t-1$.

By merging the identical terms for $\Rin \cup \Rout$, we have
\begin{align}
\lefteqn{
\Przero \left(  |A_t| = 1 \right)
}\\
&\ge \sum_{k = -t}^t \sum_{j =1}^N q_{j,t-1}(k+1) \Przero \left( 
  \min_{i \neq j} \left\{ \frac{\varepsilon}{4 N} L_{i,t-1}  + Z_{i,t-1} \right\} > \frac{\varepsilon}{4 N} L_{j,t-1}  + k + 1
\right)\hspace{-0.3em}\left(1 - Q_{t-1,k+1})\right)\hspace{-0.3em}-\frac{1}{N t}
\end{align}

Let 
\[
S_t = \left\{
j \in [N]: 
 \frac{\varepsilon}{4 N} L_{j,t-1}  + Z_{j,t-1}  
 =
  \min_{i \in [N]} \left\{ \frac{\varepsilon}{4 N} L_{i,t-1}  + Z_{i,t-1} \right\}
\right\}
\]
be the set of perturbed leaders at round $t$.
We have, 
\begin{align}
\sum_{j =1}^N \underbrace{q_{j,t-1}(k+1)}_{\Przero[k <Z_{j,t-1} \leq k+1]} 
\underbrace{
\Przero \left(
  \min_{i \ne j} \left\{ \frac{\varepsilon}{4 N} L_{i,t-1}  + Z_{i,t-1} \right\} \geq \frac{\varepsilon}{4 N} L_{j,t-1}  + k + 1
\right) 
}_{\ge \Przero\left[j \in S_t \,\middle|\, k < Z_{j,t-1} \leq k+1\right]}
\hspace{-10em}
\nonumber\\
&\ge \sum_j \Przero[j \in S_t, k<Z_{j,t-1} \leq k+1]\nonumber\\
&\ge \Przero \left[
  \exists j\in S_t: k < Z_{j,t-1} \leq k + 1
\right]\nonumber\\
&\ge \Przero \left[
  k < \min_{j\in S_t} Z_{j,t-1} \leq k + 1
\right],\label{ineq_transmergest}
\end{align}
and thus, we obtain 
\begin{align}
\lefteqn{
\Przero \left(  |A_t| = 1 \right)
}\\
&\ge 
\sum_{k=-t}^t \sum_{j =1}^N  q_{j,t-1}(k+1) \Przero \left(
  \min_{i \ne j} \left\{ \frac{\varepsilon}{4 N} L_{i,t-1}  + Z_{i,t-1} \right\} \geq \frac{\varepsilon}{4 N} L_{j,t-1}  + k + 1
\right)
\left(1 - Q_{t-1,k+1}\right)-\frac{1}{N t}
\\
&\ge 
\sum_{k=-t}^t \Przero \left[
  k < \min_{j\in S_t} Z_{j,t-1} \leq k + 1
\right]
\left(1 - Q_{t-1,k+1}\right) -\frac{1}{N t}.
\text{\ \ \ \ (by \eqref{ineq_transmergest})}
\label{eq:ali-clarification-2}
\end{align}

By using this, we have 
\begin{align}
\lefteqn{
\Przero[ |A_t| > 1 ]
}\\
&= 1 - \Przero[ |A_t| = 1 ]\\
&\le 1 - \sum_{k=-t}^{t}
\Przero\left[
  k < \min_{j\in S_t} Z_{j,t-1} \leq k + 1
\right]
\left(1 - Q_{t-1,k+1}\right) + \frac{1}{N t}\\
&= \sum_{k=-t}^{t}
\Przero\left[
  k < \min_{j\in S_t} Z_{j,t-1} \leq k + 1
\right] Q_{t-1,k+1} + \frac{1}{N t}\\
&\leq \frac{C_1 \log (t-1)}{\sqrt{\lPoi (t-1)}}
+ \frac{1}{N t}
+
\sum_{k=-t}^{t}
\Przero\left[
  k < \min_{j\in S_t} Z_{j,t-1} \leq k+1
\right]
\frac{C_2 |k+1|}{\lPoi (t-1)}
\text{\ \ \ \ (by definition of $Q_{t,k+1}$)}
\label{eq:ali-clarification-8}\\
&\le
\frac{C_1 \log (t-1)}{\sqrt{\lPoi (t-1)}}
+ \frac{1}{N t}
+
\frac{C_2}{\lPoi (t-1)}
\Ezero\left[
\left|\min_{j\in S_t} Z_{j,t-1} \right|
\right]
+ \frac{C_2}{\lPoi (t-1)}\\
&\text{\ \ \ \ \ \ (by $\sum_{k=-t}^{t}
\Przero\left[
  k < \min_{j\in S_t} Z_{j,t-1} \leq k+1
\right]
 = 1$)}\nonumber\\
&\le
\frac{C_1 \log (t-1)}{\sqrt{\lPoi (t-1)}}
+ \frac{1}{N t}
+
\frac{C_2}{\lPoi (t-1)}
\Ezero\left[
\left|\min_{j\in S_t} Z_{j,t-1} \right|
\right]
+ \frac{C_2}{\sqrt{\lPoi (t-1)}}
\text{\ \ \ \ (by \eqref{ineq_tmin} implies $\lPoi t \ge 1$)}
\\
&=
\frac{C_1 \log (t-1)}{\sqrt{\lPoi (t-1)}}
+ \frac{1}{N t}
+
\frac{C_2}{\lPoi (t-1)}
\Ezero\left[
  \left| \max_{j\in S_t} -Z_{j,t-1} \right| 
\right]
+ \frac{C_2}{\sqrt{\lPoi (t-1)}}
\label{eq:ali-clarification-11}\\
&\leq
\frac{C_1 \log (t-1)}{\sqrt{\lPoi (t-1)}}
+ \frac{1}{N t}
+
\frac{4 C_2}{\lPoi (t-1)} \sqrt{ \lPoi (t-1) \log N}
+ \frac{C_2}{\sqrt{\lPoi (t-1)}}
\label{eq:ali-clarification-3}\\
&\text{\ \ \ \ (by Lemma~\ref{lemma:E-max-abs} with $\babs=\lPoi, q=\lPoi$)} \nonumber
\end{align}
\end{proof} %thm:lead-pack-prob-bound

\section{Bounds on Poisson binomial distributions}\label{sec_poisson}

In this section, we consider the tail and mode of a Poisson binomial distribution. This distribution has one (or two consecutive) modes and is unimodal (i.e., before the left mode it is increasing and after the mode it is decreasing). The modes are within one from the mean.
We start with the well-known results on the tail of the binomial distribution, which is a special case of a Poisson binomial distribution with homogeneous parameters.
In particular, we denote $\mathrm{BIN}(t, \theta)$ to represent a binomial distribution, which is a count sum of $t$ Bernoulli random variables with a common mean $\theta$.

The following lemma lower bounds the tail of the binomial distribution.
\begin{lemma}{\rm (Binomial tail lower bound)}\label{lem_binom_tail_lower}
Consider a random variable that is drawn from a binomial distribution  $Y \sim \mathrm{BIN}(t, \theta)$. 
Let $k \ge 0$ be an nonnegative integer.
Then, 
\begin{equation}
\mathbb{P}\left[
Y \le k
\right] \ge \frac{1}{\sqrt{2t}}\exp\left(- t \cdot d\left(\frac{k}{t} \,\middle\|\, \theta\right)\right),
\end{equation} 
where $d(p \,\|\, q)
 := p\log(p/q) + (1-p)\log((1-p)/(1-q))$ is the Bernoulli KL divergence.
\end{lemma}
A proof of Lemma \ref{lem_binom_tail_lower} is found in Lemma 4.7.2 in \citep{ash1990information}.

Let us consider the Poisson binomial distribution, which generalizes the binomial distribution. 
Let $\bm{\theta} = (\theta_1, \theta_2, \dots, \theta_t)$ and $\mathrm{PBIN}(\bm{\theta})$ be the distribution of the sum of $\mathrm{Bernoulli}(\theta_1) + \dots + \mathrm{Bernoulli}(\theta_t)$. 
Let $\bar{\theta} = (1/t)\sum_{s \le t} \theta_s$ be the mean. 
We use $\locallPoi, \localuPoi$ to represent the minimum and maximum of the parameters $(\theta_s)_{s \in [t]}$, respectively. 

For $k \le t$, let $p_t(k) = \mathbb{P}[Y = k]$ for $Y \sim \mathrm{PBIN}(\bm{\theta})$. 
Let $b_t(k) = p_t(k-1)/ p_t(k)$ be the ratio of two consecutive $p_t(k)$'s. By letting $p_t(k) = 0$ for $k < 0$, $b_t(k)$ is well-defined for $k = \{0,1,\dots,t\}$. Let $P_t(k) = \sum_{k'=0}^k p_t(k')$ be the corresponding CDF. Let $P_t^c(k) = 1 - P_t(k)$. The overall goal of the rest of this section is to provide a tail bound for $\mathrm{PBIN}(\bm{\theta})$.

First, a loose upper bound is easy since each Bernoulli trial is bounded. Namely,
\begin{lemma}\label{lem_hoeffding}{\rm (Hoeffding's inequality)}
Let $Y \sim \mathrm{PBIN}(\bm{\theta})$. Then, for any $k > 0$, we have
\begin{align}
\mathbb{P}[
Y - \bar{\theta} t \ge k 
] &\le \exp\left(-\frac{2k^2}{t}\right)\\
\mathbb{P}[
Y - \bar{\theta} t \le -k 
] &\le \exp\left(-\frac{2k^2}{t}\right).
\end{align}
\end{lemma}
In the following, we focus on obtaining a sharper upper bound and a corresponding lower bound on the tail volume. 
We will derive the bound of the form 
\[
\mathbb{P}[
|Y - \bar{\theta} t| \ge k 
] \le \exp\left(-C \frac{k^2}{\bar{\theta} t}\right)
\]
for some universal constant $C>0$, which is sharper by the ${\bar{\theta}}^{-1}$ factor in the exponent.
Roughly speaking, when $\theta_1,\theta_2,\dots,\theta_t$ are very close to zero, the tail is sharper.

At a high level, the following subsection bounds a tail probability of a Poisson binomial distribution by the corresponding quantity of a binomial distribution, where the tail bounds are well-known. To do so, we use 
the Separation Lemma (Lemma \ref{lem_operation}) that essentially states that increasing the variance of the parameters $\bar{\theta}$ in a particular way always makes the tail thinner.

\subsection{Lemmas for tail bounds}

\begin{lemma}\label{lem_decrratio}
The ratio $b_t(k)$ is strictly increasing in $k$. \end{lemma}
The proof is found at Section 4.9 [1] of \cite{pollard_notebook_2021}. 
This implies the unimodality of Poisson binomial.

\begin{lemma}\label{lem_mode}
$\mathrm{PBIN}(\bm{\theta})$ has one or two consecutive modes. Let $m_L \le m_R$ be modes; in the case of a single mode, $m_L = m_R$, otherwise $m_R = m_L + 1$.
We have $|m - \bar{\theta}| \le 1$ for $m \in \{m_L, m_R\}$. 
Moreover, this fact combined with Lemma \ref{lem_decrratio} implies that 
\[
p_t(0) < p_t(1) < \dots < 
p_t(m_L) = p_t(m_R) >
\dots > p_t(t).
\]
\end{lemma}

\begin{figure}[h]
\hspace{-16em}\includegraphics[width=1\textwidth]{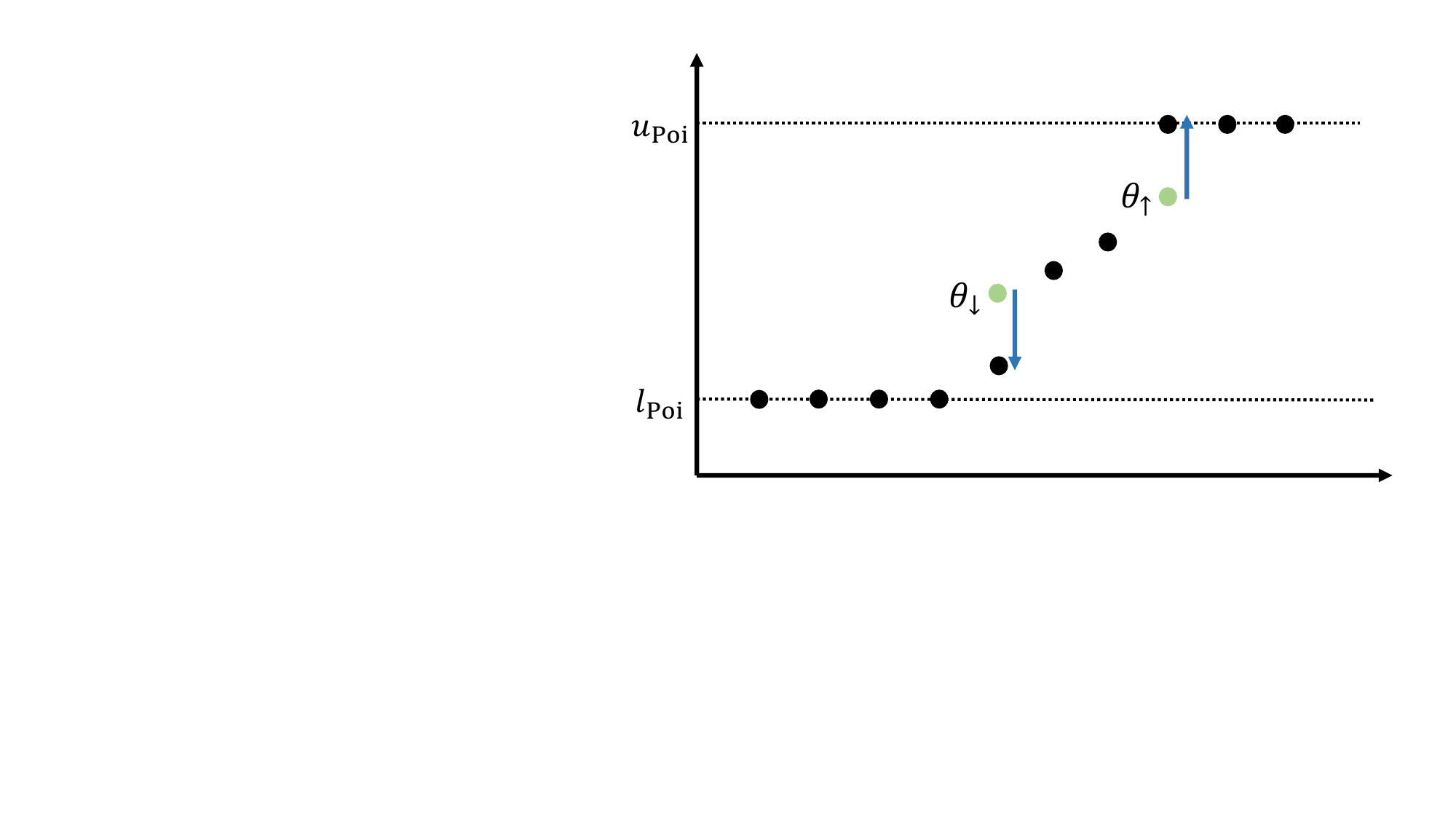}
\vspace{-8em}
  \caption{Illustration of Lemma \ref{lem_operation}. We move two parameters that were originally $\theta_\downarrow$ and $\theta_\uparrow$ (green dots) towards the direction of the blue arrows until one of them hits $\lPoi$ or $\uPoi$. The two nodes move exactly the same distance so that the summation of the parameters are preserved.}
  \label{fig:separation}
\end{figure}

\begin{lemma}{\rm (Separation Lemma)}\label{lem_operation}
Let $\bm{\theta} = (\theta_1, \theta_2, \dots, \theta_t)$ such that $\theta_1 \le \theta_2 \le \dots \le \theta_t$ and each $\theta_s \in [\locallPoi, \localuPoi] \subseteq [0, 1]$. Assume that there exist at least two indices such that $\theta_i, \theta_j \in (\locallPoi,\localuPoi)$.
Let $\downarrow, \uparrow$ be the minimum (resp. maximum) index such that $\theta_{\downarrow} > \locallPoi$ (resp. $\theta_{\uparrow} < \localuPoi$). 
Consider an operation that takes $\bm{\theta}$ and returns 
\begin{equation}
\bm{\theta}' \hspace{-0.2em} = \hspace{-0.2em}
\begin{cases}
    (\theta_1, \theta_2, \dots, \theta_{\downarrow-1}, \underbrace{\theta_{\downarrow}+\theta_{\uparrow}-\localuPoi}_{\text{position of $\theta_{\downarrow}$}}, \theta_{\downarrow+1},\dots, \theta_{\uparrow-1}, \underbrace{\localuPoi}_{\text{position of $\theta_{\uparrow}$}},\theta_{\uparrow+1},\dots,\theta_t) & \text{if } \theta_\downarrow - \locallPoi \geq \localuPoi - \theta_\uparrow \\
    (\theta_1, \theta_2, \dots, \theta_{\downarrow-1},    \underbrace{\locallPoi}_{\text{position of $\theta_{\downarrow}$}}, \theta_{\downarrow+1},\dots, \theta_{\uparrow-1}, \underbrace{\theta_{\downarrow}+\theta_{\uparrow}-\locallPoi}_{\text{position of $\theta_{\uparrow}$}},\theta_{\uparrow+1},\dots,\theta_t) & \text{otherwise}
\end{cases}
\end{equation}
Namely, we add $\locallPoi$ or $\localuPoi$ while preserving the summation, which is illustrated in Figure \ref{fig:separation}. On this operation, the following hold:
\begin{enumerate}
\item The sum of $\bm{\theta}$ and $\bm{\theta}'$ are the same. 
\item Let $k \le t\bar{\theta} - 1 \le m_L$. 
Let $p_t(k, \bm{\theta})$ be the probability of $k$ successes among $t$ coins with probabilities $\bm{\theta} = (\theta_1,\theta_2,\dots,\theta_t)$. 
Let $P_t(k, \bm{\theta}) = \sum_{k' \le k} p_t(k', \bm{\theta})$. Then,
\begin{equation}\label{ineq_tailbelow}
P_t(k, \bm{\theta}) \ge P_t(k, \bm{\theta}').
\end{equation}
\item Similarly, Let $k \ge t\bar{\theta} + 1 \le m_R$. 
Then,
\begin{equation}\label{ineq_tailover}
P_t(k, \bm{\theta}) \le P_t(k, \bm{\theta}').
\end{equation}

In other words, \eqref{ineq_tailbelow} and \eqref{ineq_tailover} state that this operation makes the tail thinner.
\end{enumerate}
\end{lemma}
\begin{proof}
We only derive the second bullet point (i.e., \eqref{ineq_tailbelow}) because the first point is trivial. Regarding the third point, inequality \eqref{ineq_tailover} can be derived by the same discussion.
The core idea of this proof is derived from Lemma 9 in \citep{jorgensen_team_2018}.
Let $\bm{\theta}_{\setminus}$ be the set of the $t-2$ parameters where $\theta_{\downarrow}$ and $\theta_{\uparrow}$ are excluded.

Let $\varepsilon = \theta_{\downarrow} - \theta_{\downarrow}' = \theta_{\uparrow}' - \theta_{\uparrow} \ge 0$. 
Let $\gamma = \theta_{\uparrow} - \theta_{\downarrow} \ge 0$.
Let $\Delta_1(k) = p_t(k, \bm{\theta}_{\setminus}) - p_t(k-1, \bm{\theta}_{\setminus})$.
Let $\Delta_2(k) = p_t(k, \bm{\theta}_{\setminus}) - 2 p_t(k-1, \bm{\theta}_{\setminus}) + p_t(k-2, \bm{\theta}_{\setminus})$.

In the following, we show that
\begin{align}\label{ineq_binom_snd}
p_t(k, \bm{\theta}) - p_t(k, \bm{\theta}')
&= \varepsilon (\varepsilon + \gamma) \Delta_2(k)\\
P_t(k, \bm{\theta}) - P_t(k, \bm{\theta}')
&= \varepsilon (\varepsilon + \gamma) \Delta_1(k).
\label{ineq_binom_fst}
\end{align}

We first derive Eq.~\eqref{ineq_binom_snd}:
\begin{align}
p_t(k, \bm{\theta}') = 
\theta_{\downarrow}' \theta_{\uparrow}' p_t(k-2, \bm{\theta}_{\setminus}) +
\left( \theta_{\downarrow}' (1 - \theta_{\uparrow}') + (1 - \theta_{\downarrow}') \theta_{\uparrow}'\right) p_t(k-1, \bm{\theta}_{\setminus}) +
(1 - \theta_{\downarrow}') (1 - \theta_{\uparrow}') p_t(k, \bm{\theta}_{\setminus}).\label{ineq_binom_snd_three}
\end{align}
Here, each term Eq.~\eqref{ineq_binom_snd_three} is transformed as
\begin{align}
\theta_{\downarrow}' \theta_{\uparrow}' p_t(k-2, \bm{\theta}_{\setminus})
& = (\theta_{\downarrow} - \varepsilon) (\theta_{\uparrow} + \varepsilon) p_t(k-2, \bm{\theta}_{\setminus})\\
& = (\theta_{\downarrow} \theta_{\uparrow} - \varepsilon (\varepsilon + \gamma)) p_t(k-2, \bm{\theta}_{\setminus})\\
\end{align}
\begin{align}
\lefteqn{
\left( \theta_{\downarrow}' (1 - \theta_{\uparrow}') + (1 - \theta_{\downarrow}') \theta_{\uparrow}'\right) p_t(k-1, \bm{\theta}_{\setminus}) 
}\\
& =
\left( (\theta_{\downarrow} - \varepsilon) (1 - \theta_{\uparrow} - \varepsilon) + (1 - \theta_{\downarrow} + \varepsilon) (\theta_{\uparrow} + \varepsilon)\right) p_t(k-1, \bm{\theta}_{\setminus}) 
\\
&=
\left( \theta_{\downarrow} (1 - \theta_{\uparrow}) + 2 \varepsilon (\varepsilon + \gamma)\right) p_t(k-1, \bm{\theta}_{\setminus}) 
\end{align}
\begin{align}
(1 - \theta_{\downarrow}') (1 - \theta_{\uparrow}') p_t(k, \bm{\theta}_{\setminus})
&=
(1 - \theta_{\downarrow} + \varepsilon) (1 - \theta_{\uparrow} - \varepsilon) p_t(k, \bm{\theta}_{\setminus})\\
&=
\left(
(1 - \theta_{\downarrow}) (1 - \theta_{\uparrow})
- \varepsilon (\varepsilon + \gamma)
\right) p_t(k, \bm{\theta}_{\setminus})\\
\end{align}
Adding these terms yields Eq.~\eqref{ineq_binom_snd}.

Eq.~\eqref{ineq_binom_fst} is derived as follows:
\begin{align}
P_t(k, \bm{\theta}) - P_t(k, \bm{\theta}')
&= \sum_{k' \le k} \left(p_t(k', \bm{\theta}) - p_t(k', \bm{\theta}')\right)\\
&= \sum_{k' \le k} \varepsilon (\varepsilon + \gamma) \Delta_2(k')
\text{\ \ \ \ (by Eq.~\eqref{ineq_binom_snd})}
\\
&= \varepsilon (\varepsilon + \gamma) \Delta_1(k)
\text{\ \ \ \ (by $p_t(-2, \bm{\theta}_{\setminus}) = p_t(-1, \bm{\theta}_{\setminus}) = 0$)}
.
\end{align}

Here, $\Delta_1(k) = p_t(k, \bm{\theta}_{\setminus}) - p_t(k-1, \bm{\theta}_{\setminus})$, and Lemma \ref{lem_decrratio} implies that $p_t(k, \bm{\theta}_{\setminus})$ as a function of $k$ is strictly increasing before the first mode of the Poisson binomial distribution, which implies $\Delta_1(k) > 0$ for $k$ before the mode, which completes the proof. 

Note that the properties we used during the proof above are $\varepsilon, \gamma \ge 0$ and $\Delta_1(k) > 0$. \eqref{ineq_tailover} can be derived by using $\varepsilon, \gamma \ge 0$ and $\Delta_1(k) < 0$ for the corresponding value of $k$. 
\end{proof} %Proof of Lemma \ref{lem_operation}

\begin{lemma}\label{lem_operation_repeat_ub}
Let $\bm{\theta}$ be such that $\theta_i \in [0, 1]$ for all $i\in[t]$. 
Let $\bm{\theta}_{unif} = (\bar{\theta},\bar{\theta},\dots,\bar{\theta})$.
Then, by applying the operation of Lemma \ref{lem_operation} on $\bm{\theta}_{unif}$ for $t-1$ times, we obtain $\bm{\theta}$.
\end{lemma}
\begin{proof}
Let $\bm{\theta}_{\mathrm{cur}}$ be the current vector, which starts with $\bm{\theta}_{unif}$. For ease of discussion, we assume values of $\bm{\theta}$ to be distinct, but that discussion here holds even if there are duplicated values.
We apply the operation of Lemma \ref{lem_operation} with $\locallPoi = \min\{\theta_i: \theta_i \notin \bm{\theta}_{\mathrm{cur}}\}$ and $\localuPoi = \max\{\theta_i: \theta_i \notin \bm{\theta}_{\mathrm{cur}}\}$.
Each operation decrements the number of element $\theta_i$ that is not included in $\bm{\theta}_{\mathrm{cur}}$, and thus we obtain $\bm{\theta}$ after applying this operation $t-1$ times.
\end{proof} %lem_operation_repeat_ub
We use Lemma \ref{lem_operation_repeat_ub} to establish an upper bound.

\begin{lemma}\label{lem_operation_repeat_lb}
Let $\bm{\theta}$ be such that $\theta_i \in [\locallPoi, \localuPoi]$ for all $i\in[t]$.
By repeating the operation of Lemma \ref{lem_operation} with $\locallPoi,\localuPoi$ on such a $\bm{\theta}$ for $t-1$ times, we have $\bm{\theta}'$ such that
\[
\bm{\theta}' = (\locallPoi, \locallPoi, \dots, \locallPoi, \underbrace{v}_{\text{at most one value $v \in (\locallPoi,\localuPoi)$}}, \localuPoi, \dots, \localuPoi, \localuPoi).
\]
\end{lemma}

The proof of Lemma \ref{lem_operation_repeat_lb} is trivial.
We use Lemma \ref{lem_operation_repeat_lb} to establish a lower bound.

\subsection{Tail upper bound}

The following bound is sharper than the Hoeffding bound (Lemma \ref{lem_hoeffding}).
\begin{lemma}{\rm (Tail Upper Bound of a Poisson Binomial Distribution)}
\label{lem_tail_upper}
Let $\bm{\theta} = (\theta_1, \ldots, \theta_t)$ be a set of $t$ Poisson binomial parameters with mean $\bar{\theta} = \frac{1}{t}\sum_{s=1}^t \theta_s$ 
such that $\bar{\theta} \leq \frac{1}{2}$. 
Let $k$ be a nonnegative integer that $k \in [0, t\bar{\theta} - 1]$. Then,
\begin{equation}\label{ineq_tail_upper_binomial}    
P_t(k, \bm{\theta}) \le \exp\left(-t \cdot d \left(k/t \,\middle\| \, \bar{\theta}\right)\right).
\end{equation}
Moreover, we have
\begin{equation}\label{ineq_tail_upper_binomial_approx}
P_t(k, \bm{\theta}) \le \exp\left(
- \frac{t}{2\bar{\theta}} \left( \frac{k}{t} - \bar{\theta} \right)^2 
\right).
\end{equation}

Similarly, for $k \ge t \bar{\theta} + 1$, we have
\begin{equation}\    
P_t^c(k-1, \bm{\theta}) \le \exp(-t d(k/t, \bm{\theta})) \le \exp\left(
- \frac{t}{2\bar{\theta}} (k/t - \bar{\theta})^2 
\right).
\end{equation}
\end{lemma}
One can see that this bound is sharper by the factor $(\bar{\theta})^{-1}$.
\begin{proof}
We consider the case of  $k \le t\bar{\theta} - 1$. 
Eq.~\eqref{ineq_tail_upper_binomial} is derived by the following steps. 
First, Lemma \ref{lem_operation_repeat_ub} states that if we apply the transformation of Lemma \ref{lem_operation} to $\bm{\theta}_{unif} = (\bar{\theta},\bar{\theta},\dots,\bar{\theta})$ for $t-1$ times we obtain $\bm{\theta}$. The transformation always makes the tail thinner, and thus
\[
P_t(k, \bm{\theta}) \le P_t(k, \bm{\theta}_{unif}).
\]

Since $\mathrm{PBIN}(\bm{\theta}_{unif}) = \mathrm{BIN}(t, \bar{\theta})$ is a binomial distribution, we can apply a Chernoff bound on the corresponding binomial distribution, which implies
\[
P_t(k, \bm{\theta}_{unif})
\le \exp\left(-t \cdot d \left(k/t \,\middle\| \, {\bar{\theta}}\right)\right).
\]

Eq.~\eqref{ineq_tail_upper_binomial_approx} is derived from Eq.~\eqref{ineq_tail_upper_binomial} as follows.

\begin{align}
 d\left(k/t\, \middle\| \, \bar{\theta}\right) 
&= \int_{\frac{k}{t}}^{\bar{\theta}} \frac{\diff}{\diff x}\left( \, d\left(k/t\, \middle\|\, x\right) \right)\, \diff x\nonumber\\
&= \int_{k/t}^{\bar{\theta}} \frac{x - (k/t)}{x(1-x)} \diff x \nonumber\\
&\ge \int_{k/t}^{\bar{\theta}} \frac{x - (k/t)}{{\bar{\theta}}(1-{\bar{\theta}})} \diff x \tag{$\bar{\theta} \le 1/2$}\nonumber\\
&=  \frac{({\bar{\theta}} - k/t)^2}{2 {\bar{\theta}}(1-{\bar{\theta}})} \nonumber\\
&\ge  \frac{({\bar{\theta}} - k/t)^2}{2 \bar{\theta}}. \nonumber
\end{align}

The corresponding bound for the case of $k \ge t\bar{\theta} + 1$ can be derived by the same discussion.
\end{proof}

\subsection{Tail lower bound}
\label{subsec_taillower}

\begin{lemma}{\rm (Tail Lower Bound of a Poisson Binomial Distribution)} \label{lem_tail_lower}
Let $\bm{\theta} = (\theta_1, \ldots, \theta_t)$ be a set of $t \ge 6$ Poisson binomial parameters with mean $\bar{\theta} = \frac{1}{t}\sum_{s=1}^t \theta_s$ such that all parameters lie in $\locallPoi \le \theta_s \le 1/2$. 
Let $k$ be an nonnegative integer such that $k \in [t\bar{\theta} + 1 - \locallPoi t/8, t\bar{\theta} - 1]$.
Then,
\begin{equation}\label{ineq_tailower_sq}
P_t(k, \bm{\theta}) \ge
\frac{1}{\sqrt{2t}}\exp\left(- \frac{36 t}{\locallPoi}
\left(
(\bar{\theta} - (k - 1)/t)^2
\right)
\right).
\end{equation}
\end{lemma}
\begin{proof}
We apply Lemma~\ref{lem_operation_repeat_lb} for $\bm{\theta}$ with $(\locallPoi,\localuPoi) = (\locallPoi,1)$, 
which gives
\begin{equation}
\bm{\theta}' = \left(\underbrace{\locallPoi, \locallPoi, \dots, \locallPoi}_{n}, v, 1, 1, \dots, 1\right),
\end{equation}
where $n$ is number of $\locallPoi$ above and $v \in [\locallPoi, 1]$, such that 
\[
P_t(k, \bm{\theta}) \ge P_t(k, \bm{\theta}').
\] 
Note that since $\bar{\theta}\leq 1/2$, and $\vert \bm{\theta}\vert_1 = \vert \bm{\theta}'\vert_1$, we have $n \ge t/2-1 \ge t/3$ for $t \ge 6$. Additionally, by definition, 
\begin{equation}\label{ineq_nfourone}
t \bar{\theta} \le n \cdot \locallPoi + (t-n) \cdot 1 .
\end{equation}

Since $\mathrm{Bernoulli}(1)$ is deterministically $1$, and $\mathrm{Bernoulli}(v) \leq 1$, the sufficient condition for $\mathrm{PBIN}(\bm{\theta}')$ to have value equal or smaller than $k$ is that the binomial distribution $\mathrm{BIN}(n, \locallPoi)$ has value at most $k - (t - n)$. Therefore, we have 
\begin{align*}
    P_t(k, \bm{\theta}') &\geq \Pr{\left\{\mathrm{BIN}(n,\locallPoi) \leq k - (t-n)\right\}}\\
    &\geq \Pr{\left\{\mathrm{BIN}(n,\locallPoi) \leq k - (t-n)-1\right\}},
\end{align*}
which can be further lower bounded by applying Lemma~\ref{lem_binom_tail_lower} as follows.
    \begin{align}
        &\Pr{\left\{\mathrm{BIN}(n,\locallPoi) \leq k - (t-n)-1\right\}} \nonumber\\
&\ge \frac{1}{\sqrt{2n}}\exp\left(- n \cdot d\left(\frac{k - (t - n)-1}{n} \, \middle \| \, \locallPoi\right)\right) \tag{by Lemma~\ref{lem_binom_tail_lower}}\nonumber\\
&\ge \frac{1}{\sqrt{2t}}\exp\left(- t \cdot d\left(\frac{k - (t - n)-1 }{n} \, \middle \| \, \locallPoi\right)\right)
\tag{$n \leq t$}\nonumber\\
&\ge \frac{1}{\sqrt{2t}}\exp\left(- t\cdot  d\left(\frac{k + n\locallPoi - t \bar{\theta} -1}{n} \, \middle \| \, \locallPoi\right)\right)
\tag{by Eq.~\eqref{ineq_nfourone}}\nonumber\\
&\ge \frac{1}{\sqrt{2t}}\exp\left(- t \cdot d\left(\frac{k - t \bar{\theta} -1}{t/3} + \locallPoi \, \middle \| \, \locallPoi\right)\right)
\text{\ \ \ \ \ (by $n \ge t/3$)}. \label{eq:bin-n-lower-bonud}
    \end{align}

To further lower bound this, let
\begin{align}
x &:= \locallPoi \leq \min_{s \in [t]} \theta_s \leq 1/2\\
y &:= - \frac{k - t \bar{\theta} - 1}{t/3}.
\end{align}
We have 
\begin{align}
x - y 
&= \frac{k - t \bar{\theta} - 1}{t/3} + \locallPoi \ge -\frac{\locallPoi t/8}{t/3} + \locallPoi  = \frac{5}{8}\locallPoi \ge \frac{1}{4}\locallPoi,\\
x - y &\le \frac{-2}{t/3} + \locallPoi \le \locallPoi \le 1/2,
\end{align}

By using the integration formula of KL divergence (cf., Lemma 11 of \cite{komiyama2024rate}) 
for $x,y$ such that $x,y \ge 0$, $x-y \le 1/2$, we have 
\begin{align}
d\left(x-y \, \middle \|\, x\right) 
&= \int_0^{y} \frac{z}{(x-y+z)(1-x+y-z)} \diff z \nonumber\\
&\leq   \frac{1}{(x-y)(1-x+y)} \int_0^{y} z \, \diff z &(x-y\leq \locallPoi \leq 1/2) \nonumber\\
&\le \frac{y^2}{2(x-y)(1-x+y)}
&\label{ineq_xyintegral}
\end{align}
and thus
\begin{align}
P_t(k, \bm{\theta}') 
&\ge 
\frac{1}{\sqrt{2t}}\exp\left(- t 
\frac{y^2}{2(x-y)(1-x+y)}
\right)
\tag{by \eqref{eq:bin-n-lower-bonud}, \eqref{ineq_xyintegral}}
\\
&\geq
\frac{1}{\sqrt{2t}}\exp\left(- t 
\frac{9 (\bar{\theta} - (k-1)/t)^2}{2 (\locallPoi/4) \cdot (1/2)}
\right)
\tag{by $\locallPoi/4 \le x-y \le 1/2$ and Definition of $y$}
\\
&=
\frac{1}{\sqrt{2t}}\exp\left(- \frac{36 t}{\locallPoi}
\left(
(\bar{\theta} - (k-1)/t)^2
\right)
\right),\label{ineq_lowertail_final}
\end{align}
which is \eqref{ineq_tailower_sq}. 
\end{proof}

\begin{lemma}{\rm (Poisson Binomial Tail Ratio)}\label{lem_poibin_tail_ratio}
Let $\bm{\theta} = (\theta_1, \ldots, \theta_t)$ be a set of $t$ Poisson binomial parameters with mean $\bar{\theta} = \frac{1}{t}\sum_{s=1}^t \theta_s$ such that all parameters lies in $\locallPoi \le \theta_s \le 1/2)$. 
Let $t \ge 24^2/\locallPoi$. 
Let $k$ be an integer in 
$[t \bar{\theta} + 2 - \locallPoi t / 24, t \bar{\theta} - 2 + \locallPoi t / 24]$. 
Then there exist universal constants\footnote{A constant is universal if it does not depend on model parameters.} $C_1, C_2 > 0$ such that the following holds:
\begin{align}\label{ineq_poibin_tail_ratio}
b_t(k, \bm{\theta}) 
:= \frac{p_t(k-1, \bm{\theta})}{p_t(k, \bm{\theta})}
\geq 1 
     - \frac{C_1 \log t}{\sqrt{\locallPoi t}} 
     - \frac{C_2 \left| k - t \bar{\theta} \right|}{\locallPoi t}.
\end{align}
\end{lemma}

\begin{proof}

This proof derives \eqref{ineq_poibin_tail_ratio} for each case of $k$.

\vspace{1em}
\noindent\textbf{Case $\bm{k \in [t\bar{\theta} + 2 - \locallPoi t/24, t\bar{\theta} - 1]}$:}

Let $\Delta := t \bar{\theta} - k > 0$.
Let $k_2 := \bigl\lceil k - \Delta - \sqrt{\locallPoi t} - 1 \bigr\rceil$, which is nonnegative. 
We have
\begin{align*}
k_2 
= \bigl\lceil k - \Delta - \sqrt{\locallPoi t} - 1 \bigr\rceil 
&\geq k - \Delta - \sqrt{\locallPoi t} - 1 \\
&\geq k - \Delta - \frac{\locallPoi t}{24} - 1 
&\tag{by $t \geq 24^2 / \locallPoi$}\\
&= 2k - t \bar\theta - \frac{\locallPoi t}{24} - 1 \tag{Definition of $\Delta$}\\
&\geq t \bar{\theta} - \frac{\locallPoi t}{8} + 3 
\tag{by $k \geq t \bar{\theta} + 2 - \locallPoi t / 24$} ,
\end{align*}
and thus Eq.~\eqref{ineq_tailower_sq} in Lemma~\ref{lem_tail_lower} can be applied to $k_2$, which implies that
\begin{align}
\sum_{k'=0}^{k_2} p_t(k', \bm{\theta}) 
&\geq 
\frac{1}{\sqrt{2t}} 
\exp \left( -36 \frac{t}{\locallPoi}
         \left(
             (\bar{\theta} - (k_2 - 1) / t)^2
         \right)
     \right)\nonumber\\
&\geq 
\frac{1}{\sqrt{2t}} \exp \left( - (36 \cdot 4) \frac{t}{\locallPoi}
                        (\bar{\theta} - k_2 / t)^2
                 \right) 
\tag{by $t \bar{\theta} - k_2 \geq 1$} \\
&\geq 
\frac{1}{t} \exp \left( -144 \frac{t}{\locallPoi}
                        (\bar{\theta} - k_2 / t)^2
                 \right) , \nonumber
\end{align}
which, combined with the fact that $p_t(k', \bm{\theta})$ is increasing up to $m_L$ implies that
\begin{align}\label{ineq_ptktwolower}
p_t(k_2, \bm{\theta}) 
> \frac{1}{k_2+1} \sum_{k'=0}^{k_2} p_t(k', \bm{\theta}) 
> \frac{1}{t} 
  \cdot \frac{1}{t}\exp\left( -144 \frac{t}{\locallPoi}
                              (\bar{\theta} - k_2/t)^2
                       \right) .
\end{align}

By definition, 
\begin{align*}
\frac{p_t(k_2, \bm{\theta})}{p_t(k, \bm{\theta})} 
&= \prod_{k'=k_2}^{k-1} \frac{p_t(k', \bm{\theta})}{p_t(k'+1, \bm{\theta})} \\
&\leq \left( \frac{p_t(k-1, \bm{\theta})}{p_t(k, \bm{\theta})} \right)^{k - k_2}
\tag{by Lemma \ref{lem_decrratio}}
\end{align*}
and thus
\begin{align}
\left( 
    \frac{p_t(k-1, \bm{\theta})}{p_t(k, \bm{\theta})}
\right)^{k - k_2}
\geq \frac{p_t(k_2, \bm{\theta})}{p_t(k, \bm{\theta})}
\geq p_t(k_2, \bm{\theta}).
\label{ineq_ptktwolower_power}
\end{align}

Combining \eqref{ineq_ptktwolower} and \eqref{ineq_ptktwolower_power} implies that
\begin{align}
\lefteqn{
    \frac{p_t(k-1, \bm{\theta})}{p_t(k, \bm{\theta})} 
} \nonumber\\
&\geq \left(
          \frac{1}{t^2}
          \exp\left( -144 \frac{t}{\locallPoi}
                     (\bar{\theta} - k_2 / t)^2
              \right)
      \right)^{1/(k - k_2)} \nonumber \\
&= \exp\left( -\frac{1}{k - k_2}
              \left( 2 \log t + 144 \frac{t}{\locallPoi}
                     (\bar{\theta} - k_2/t)^2
              \right)
       \right) \nonumber\\
&\geq 1 
      - \frac{1}{k - k_2}
        \left(
            2 \log t + 144 \frac{t}{\locallPoi} (\bar{\theta} - k_2/t)^2
        \right)
\tag{by $e^{-x} \ge 1-x$}
\\
&\geq 1 
      - \frac{2\log t}{\Delta+\sqrt{\locallPoi t}} 
      - 144 \frac{t}{\locallPoi (\Delta + \sqrt{\locallPoi t})} (\bar{\theta} - k_2/t)^2
\tag{by $k - k_2 \ge \Delta + \sqrt{\locallPoi t}$ and $\Delta > 0$}
\\
&= 1 
   - \frac{2\log t}{\Delta+\sqrt{\locallPoi t}} 
   - 144 \frac{t}{\locallPoi (\Delta + \sqrt{\locallPoi t})} 
     \left(\frac{\bar\theta t - k_2}{t}\right)^2 \nonumber\\
&\geq 1 
      - \frac{2\log t}{\Delta+\sqrt{\locallPoi t}} 
      - 144 \frac{t}{\locallPoi (\Delta + \sqrt{\locallPoi t})}
        \left(\frac{2 \Delta + \sqrt{\locallPoi t} + 1}{t}\right)^2 
\tag{only dealing with the ceiling from $k_2$} \\
&= 1 
   - \frac{2\log t}{\Delta+\sqrt{\locallPoi t}} 
   - 576 \frac{\left(\Delta + \frac{1}{2} \sqrt{\locallPoi t} + \frac{1}{2}\right)^2}{\locallPoi t (\Delta + \sqrt{\locallPoi t})} \nonumber\\
&\geq 1 
      - \frac{2\log t}{\Delta+\sqrt{\locallPoi t}} 
      - 576 \frac{\Delta + \frac{1}{2} \sqrt{\locallPoi t} + \frac{1}{2}}{\locallPoi t} \tag{$\sqrt{\locallPoi t} \geq 24$}\\
&= 1 
   - \frac{2\log t}{\Delta+\sqrt{\locallPoi t}} 
   - 576 \frac{t \bar{\theta} - k + \frac{1}{2} \sqrt{\locallPoi t} + \frac{1}{2}}{\locallPoi t} \nonumber\\
&= 1 
   - \frac{2\log t}{\Delta+\sqrt{\locallPoi t}} 
   - 288 \frac{1}{\sqrt{\locallPoi t}} 
   - 576 \frac{t \bar{\theta} - k}{\locallPoi t} 
   - 288 \frac{1}{\locallPoi t} \nonumber\\   
&\geq 1 
      - \frac{2\log t}{\Delta+\sqrt{\locallPoi t}} 
      - 288 \frac{1}{\sqrt{\locallPoi t}} 
      - 864 \frac{t \bar{\theta} - k}{\locallPoi t} \tag{using $t \bar{\theta} - k \geq 1$} \\
&\geq 1 
      - \frac{290 \log t}{\sqrt{\locallPoi t}} 
      - 864 \frac{|k - t \bar{\theta}|}{\locallPoi t} \nonumber\\
      &=1 
      - \frac{C'_1 \log t}{\sqrt{\locallPoi t}} 
      -  \frac{C'_2|k - t \bar{\theta}|}{\locallPoi t} ,\label{eq:first-case-bti-bound}
\end{align}
where we define $C'_1:=290$ and $C'_2:=864$. Provided that provided that $C_1 \geq C'_1 $ and $C_2 \geq C'_2$, satisfy Eq~\eqref{ineq_poibin_tail_ratio}.

\vspace{1em}
\noindent\textbf{Case $\bm{k \in [t\bar{\theta} - 1, t\bar{\theta} + 1]}$:}

By the first case, for all integers 
$k' \in [t \bar{\theta} - 2, t \bar{\theta} - 1]$, 
we have
\begin{align*}
b_t(k', \bm{\theta}) 
\ge 1 
    - \frac{C'_1 \log t}{\sqrt{\locallPoi t}} 
    - \frac{C'_2 |k' - t \bar{\theta}|}{\locallPoi t} .
\end{align*}
Now, for any $k \in [t\bar{\theta} - 1, t\bar{\theta} + 1]$, we can and will take some $k' \in [t \bar{\theta} - 2, t \bar{\theta} - 1]$ such that $|k - k'| \leq 2$. Then we have
\begin{align*}
b_t(k, \bm{\theta}) 
&\geq b_t(k', \bm{\theta}) 
\tag{$b_t(k, \bm{\theta})$ is increasing in $k$} \\
&\geq 1 
      - \frac{C'_1 \log t}{\sqrt{\locallPoi t}} 
      - \frac{C'_2 |k' - t \bar{\theta}|}{\locallPoi t} \tag{Eq~\eqref{eq:first-case-bti-bound}}\\
&\geq 1 
      - \frac{C'_1 \log t}{\sqrt{\locallPoi t}} 
      - \frac{C'_2 \left(|k - t \bar{\theta}| + \vert k'-k\vert\right)}{\locallPoi t}  \\
      &\geq 1 
      - \frac{C'_1 \log t}{\sqrt{\locallPoi t}} 
      - \frac{C'_2 \left(|k - t \bar{\theta}| + 2\right)}{\locallPoi t}  \\
&\geq 1 
      - \frac{(C'_1 + \frac{1}{12} C'_2) \log t}{\sqrt{\locallPoi t}} 
      - \frac{C'_2 |k - t \bar{\theta}|}{\locallPoi t} 
\tag{since $\frac{\log t}{24 \sqrt{\locallPoi t}} \geq \frac{1}{\locallPoi t}$} .
\end{align*}

This again is of the form \eqref{ineq_poibin_tail_ratio} provided that $C_1 \geq C'_1 + \frac{C'_2}{12}$ and $C_2 \geq C'_2$.

\vspace{1em}
\noindent\textbf{Case $\bm{k \in  [t\bar{\theta} + 1, t \bar\theta-2 + \locallPoi t/24]}$:}

Using the fact that $b_t(k, \bm{\theta}) \geq 1$ for $k \geq m_R$, this case trivially goes through for any  $C_1, C_2 \ge 0$.

Finally, the choice $C_1 = 362$ and $C_2 = 864$ simultaneously handles all three cases.
\end{proof}

\section{ELF-X versus online extension of I-ELF} \label{app:elf-x}

Written in terms of scores $s_{i,t} \in [0, 1]$ or equivalently as losses $\loss_{i,t} := 1 - s_{i,t} \in [0, 1]$, I-ELF (extended to the online learning setting) predicts in round $t$ as follows. First, for any expert $i$ and round $t$, define $f_{i,t}$ as
\begin{align}
f_{i,t} 
= \frac{1}{N} + \frac{1}{N} \left( s_{i,t} - \frac{1}{N-1} \sum_{j \in [N] \setminus \{i\}} s_{j,t} \right)
= \frac{1}{N} \left( 1 - \loss_{i,t} + \frac{1}{K-1} \sum_{j \in [N] \setminus \{i\}} \loss_{j,t} \right) . \label{eqn:elf}
\end{align}
Next, let $i^*_t$ be a random variable that is equal to $i$ with probability $f_{i,t}$, and define the indicator random variable $Y_{i,t} = \ind{i = i^*_t}$. At the end of the game (in ``round'' $T+1$), ELF selects expert $I_{T+1}$, defined as
\begin{align*}
I_{T+1} = \argmax_{i \in [N]} \sum_{t=1}^T Y_{i,t} .
\end{align*}

It is straightforward to adapt I-ELF to the online learning setting, giving a method we may call ``Online I-ELF'': in round $t$, select expert $I_t$, defined as
\begin{align*}
I_t = \argmax_{i \in [N]} \sum_{s=1}^{t-1} Y_{i,s} .
\end{align*}
Unfortunately, Online I-ELF can get linear regret. To see this, suppose that $N = 2$ and the loss vector sequence is
\begin{align*}
\begin{pmatrix}
\loss_{1,1} & \cdots & \loss_{1,T} \\ 
\loss_{2,1} & \cdots & \loss_{2,T}
\end{pmatrix} 
= 
\begin{pmatrix}
1 & 0 & 1 & 0 & 1 & \cdots & 0 \\
0 & 1 & 0 & 1 & 0 & \cdots & 1 
\end{pmatrix} .
\end{align*}
Then for $t$ odd we have $f_{1,t} = 0$ and $f_{2,t} = 1$, which implies that $Y_{1,t} = 0$ and $Y_{2,t} = 1$. Similarly, for $t$ even we have $f_{1,t} = 1$ and $f_{2,t} = 0$, which implies that $Y_{1,t} = 1$ and $Y_{2,t} = 0$. Consequently, in any even round, we must have $I_t = 2$, while in any odd round (assuming times are broken uniformly at random) we have $I_t = 1$ with probability $\frac{1}{2}$. Overall, Online I-ELF gets regret $\frac{3 T}{4} - \frac{T}{2} = \frac{T}{4}$, which is bad.

It is likely this issue that led \citet{freeman2020no} to propose ELF-X, the online extension (as above) of a slight modification of I-ELF which involves defining $f_{i,t}$ as
\begin{align}
f_{i,t} 
= \frac{1}{N} \left( 1 - \loss_{i,t} + \frac{1}{N} \sum_{j=1}^N \loss_{j,t} \right) , \label{eqn:elf-x}
\end{align}
This algorithm avoids the previous issue by always ensuring that $f_{i,t}$ is bounded away from zero and one.

\end{document}